\documentclass{article}

  \usepackage[preprint,nonatbib]{neurips_2026}

\usepackage[utf8]{inputenc}
\usepackage[T1]{fontenc}
\usepackage[round]{natbib}
\usepackage{amsmath,amssymb,amsthm}
\usepackage{mathtools}
\usepackage{bm}  \usepackage{xcolor}
\usepackage{graphicx}
\graphicspath{{../}{../../}{../../../}{./}}
\usepackage{booktabs}
\usepackage{hyperref}
\usepackage{algorithm}
\usepackage{algorithmic}
\usepackage{subcaption}
\usepackage{multirow}

\newtheorem{theorem}{Theorem}
\newtheorem{proposition}[theorem]{Proposition}
\newtheorem{lemma}[theorem]{Lemma}
\newtheorem{corollary}[theorem]{Corollary}
\newtheorem{definition}[theorem]{Definition}
\newtheorem{remark}[theorem]{Remark}

\newcommand{\kl}{\mathrm{KL}}
\newcommand{\rlct}{\lambda}
\newcommand{\fisher}{F}
\newcommand{\reals}{\mathbb{R}}
\newcommand{\expect}{\mathbb{E}}

\DeclareMathOperator{\cov}{Cov}

 \newif\ifexpanded
\expandedfalse  

\newif\ifbodyhasproofs
\bodyhasproofsfalse

\newif\iftheoryonly
\theoryonlyfalse

\newif\ifsupp
\suppfalse
 \newif\ifanonymise
\anonymisetrue  

\newif\ifanonymousfriendly
\anonymousfriendlytrue  

\newcommand{\theorycite}{ \citet{TheoryRefNamed}}
\newcommand{\theorycitep}{ \citealp{TheoryRefNamed}}

\newcommand{\theorysrc}{ the theory paper}

\newcommand{\theorypdf}{ the theory paper}
\newcommand{\theorytag}{ theory paper}
 
  \anonymisefalse
  \expandedfalse
  \anonymousfriendlyfalse

\title{Algebraic Dead Directions in LayerNorm Transformers: A Forward-Pass-Only Diagnostic at LLM Scale}

  \author{
    Tejas Pradeep Shirodkar \\
    IIIT, Hyderabad \\
    \texttt{tejas.shirodkar@research.iiit.ac.in}
    \And
    P. J. Narayanan \\
    IIIT, Hyderabad \\
    \texttt{pjn@iiit.ac.in}
  }

\begin{document}

\maketitle

\begin{abstract}
Pretrained transformers sit near singular minima of the loss, where the Fisher information metric degenerates along dead directions: directions in parameter space along which the directional Fisher vanishes. Locating such a direction normally needs a forward pass and an eigendecomposition of activations, or a sampling-based complexity estimate; none returns a direction computable from the network's parameters alone. We give one, for LayerNorm transformers. The inverse-scale direction $\gamma^{-1}/\|\gamma^{-1}\|$ of the LayerNorm affine is an exact algebraic kernel of the post-final-norm centred activation covariance, for any input distribution, and induces a corresponding dead direction in parameter space. It is read from the LN scale parameter alone, with no forward or backward pass and no eigensolve: the cheapest dead-direction read, specific to LayerNorm. We test it on $14$ pretrained transformers ($9$ LayerNorm, $5$ RMSNorm; $160$M-$35$B; language and vision objectives). At random initialisation the predicted direction matches the measured bottom singular direction (one forward pass, direct SVD) to four decimal places on $9/9$ LayerNorm models, and is correctly absent on $5/5$ RMSNorm models, which lack the mean-subtraction projector that creates it. On the trained checkpoint the covariance eigenvalue along this direction deepens by ${\sim}10^3\times$ and further dead directions open; the random-init-to-trained gap is a one-forward-pass, per-checkpoint readout of singular structure along the predicted coordinate. Two consequences follow in closed form: the residual stream's smallest singular value is preserved block-to-block on $13/14$ transformers measured on their own input distribution, the one exception (Gemma~$4$-$31$B) a genuine dead direction the same read pinpoints; and the kernel direction's presence classifies a transformer's normalisation from the parameters alone.
 \end{abstract}

\section{Introduction}
\label{sec:intro}

\begin{figure}[!t]
    \centering
    \includegraphics[width=\textwidth]{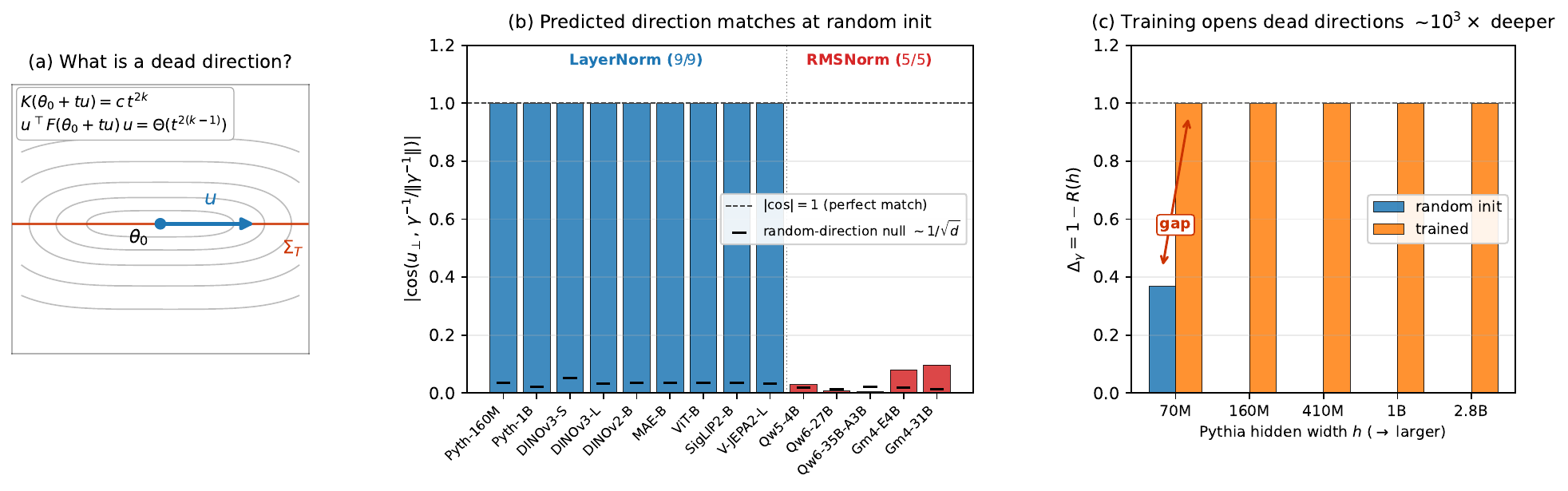}
    \caption{\textbf{The result in three pictures.}
    \textbf{(a)} A dead direction $u$ at $\theta_0 \in \Sigma_T$ is a
    unit vector along which the Fisher metric degenerates: $K(\theta_0
    + tu) = c\, t^{2k}$ vanishes faster than quadratic, so $u^\top
    F(\theta_0 + tu)\, u = \Theta(t^{2(k-1)})$. The framework of
    \theorycite{} derives such directions in closed form from a
    network's affine parameters.
    \textbf{(b)} For LayerNorm-equipped pretrained transformers the
    predicted direction is $\gamma^{-1}/\|\gamma^{-1}\|$, readable
    from the LN affine with no forward pass. At random initialisation
    $|\cos(u_\bot, \gamma^{-1}/\|\gamma^{-1}\|)| \ge 0.9999$ on $9/9$
    LN models (mean $0.99999$, centered protocol); the $5$ RMSNorm
    models sit at the random-direction null $1/\sqrt{d}$ (per-model
    markers), confirming the no-universal-kernel differential.
    \textbf{(c)} On the same Pythia LN sweep, $\Delta_\gamma := 1 -
    R(h)$ at the predicted direction sits at the finite-$N$
    Marchenko--Pastur drift on the smallest cell ($\approx 0.37$ at
    $h{=}512$, within $\pm 0.06$ of zero at $h \ge 768$) at random
    init, and saturates at $\approx 1$ across all sizes on the
    trained checkpoint, a $\sim\!10^3\times$ deepening along the
    same coordinate. The
    gap is the diagnostic: training has opened additional dead
    directions strictly below the algebraic baseline that LN's
    forward map alone establishes.}
    \label{fig:overview_headline}
\end{figure}
 
Trained transformers do not sit at isolated minima of their loss. The parameter spaces of overparameterised models contain analytic singular sets along which the Fisher information metric degenerates~\citep{Watanabe09}, and a trained network sits at one specific configuration on that singular structure. The practical question is which configuration: along which directions in parameter space is the loss flat, and what does that flatness look like at the network's actual weights? An answer in original parameter coordinates would let practitioners identify dead subspaces without sampling, gradient access, or the $10^4$--$10^6$ forward--backward passes per checkpoint that the local learning coefficient \citep[LLC;][]{LauFurmanWangMurfetWei25, HooglandWangFarrugiaRoberts24} requires. The dead-direction framework of \theorycite{} gives such an answer in original parameter coordinates, deriving a network's flat directions (dead directions) from a KFAC factorisation of its Fisher. This paper develops its LayerNorm case, where the answer is closed-form: a single dead direction read from one parameter with no forward pass, validated across pretrained transformers at LLM scale.

Singular learning theory locates these dead directions through Hironaka resolution \citep{Watanabe09}, a coordinate change that recasts the KL divergence in normal-crossing form so the local real log canonical threshold is readable. Information geometry encodes the same structure as a degeneration of the Fisher metric \citep{AmariParkOzeki06, Amari16}, but presupposes the metric is non-singular and so is silent on the singular set itself. Recent activation-spectrum work tracks $\sigma_{\min}$ \citep{Jha2026NerVE, Ettori2025EigenTrack, WangGeShu25}, weight-update SVD \citep{Xu2026SpectralEdge}, and Marchenko--Pastur deviations on pretrained weights \citep{StaatsThammRosenow24}, and the rank-collapse line \citep{DongCordonnierLoukas21, NociAnagnostidisBiggio22} characterises the same observable in pure-attention chains. None of these methods returns a specific direction in a specific architecture that is computable from the network's affine parameters alone. We give one, the framework's LayerNorm instance: a closed-form unit vector in residual-stream space, read from the LayerNorm scale parameter alone, that lies in the kernel of the post-final-norm centred activation covariance for any input distribution. The KFAC factorisation of \citet{MartensGrosse15} is the measurement primitive that makes this direction parameter-readable rather than only forward-pass-readable.

Our central architecture-side prediction admits a four-decimal test on real LLMs. Write $\gamma \in \mathbb{R}^d$ for the LayerNorm scale parameter and let $u_\bot$ denote the bottom singular direction of the post-final-norm centred activation covariance. For any LayerNorm-equipped transformer, this covariance is rank-deficient exactly along the inverse-scale direction $v^* := \gamma^{-1}/\|\gamma^{-1}\|$, a unit vector readable from the LN affine parameter alone with no forward pass required (Proposition~\ref{pred:ln_kernel}, \S\ref{sec:predictions}). For RMSNorm, lacking the mean-subtraction projector, no such universal direction exists (Proposition~\ref{pred:rmsnorm_no_kernel}). Both predictions are theorems of \theorysrc{}, restated with proof sketches in \S\ref{sec:predictions} and subjected here to a four-decimal test at LLM scale. The two propositions partition any pretrained transformer into two classes: LN models should exhibit the predicted direction algebraically; RMSNorm models should not. We test both at random initialisation and on the trained checkpoint, and treat the gap between the two as a third measurement.

The match is sharp on every applicable model. On $9$ LN-equipped pretrained transformers spanning next-token CE, supervised image classification, DINO-SSL, MAE-reconstruction, contrastive-image-text, and predictive-video-SSL objectives, the cosine alignment satisfies $|\cos(u_\bot, v^*)| \ge 0.9999$ at random initialisation (mean $0.99999$); the $5$ RMSNorm models sit at the random-direction null. To track how training deepens the kernel beyond this algebraic baseline, we use the per-block \emph{Schur ratio} $R(h) := \lambda_{\min}(A_\ell)/(A_\ell)_{u, u}$ comparing the smallest activation-covariance eigenvalue at block $\ell$ against the eigenvalue along the predicted direction $u = \gamma^{-1}/\|\gamma^{-1}\|$; its complement $\Delta_\gamma := 1 - R(h) \in [0, 1]$ measures \emph{how much deeper than} $u$ \emph{the actual bottom of the spectrum sits}. $\Delta_\gamma \approx 0$ means the predicted direction \emph{is} the bottom; $\Delta_\gamma \approx 1$ means the bottom is orders of magnitude smaller along some \emph{other} direction. On the Pythia LN sweep, random-init $\Delta_\gamma$ reads $\approx 0.37$ at $h{=}512$ (finite-$N$ Marchenko--Pastur drift $\sim 2\sqrt{h/N} - h/N$, largest at the smallest cell) and within $\pm 0.06$ of zero at $h \ge 768$; the trained checkpoint saturates at $\approx 1$ across all sizes, a $\sim\!10^3\times$ deepening along the same coordinate. The gap between random-init and trained $\Delta_\gamma$ is the diagnostic: training opens dead directions strictly deeper than the architectural baseline establishes.

\paragraph{Contributions.}
\begin{enumerate}\itemsep=3pt
\item \textbf{Algebraic baseline + gap diagnostic on $14$ pretrained transformers.} The predicted dead direction $\gamma^{-1}/\|\gamma^{-1}\|$ matches the measured bottom singular direction to four decimal places on $9/9$ LN models at random init (centered protocol; mean $|\cos| > 0.99999$), and is correctly absent on $5/5$ RMSNorm models ($|\cos| \le 0.10$ centered). On the same Pythia size sweep, the Schur ratio collapses from random-init $\Delta_\gamma \approx 0.37$ at $h{=}512$ (within $\pm 0.06$ of zero at $h \ge 768$) to $\Delta_\gamma \approx 1$ trained across all sizes, with a $3$--$18\times$ directional beat over the axis-aligned default (\S\ref{sec:experiments}, Tab.~\ref{tab:schur_pole_c}).

\item \textbf{$\sigma_{\min}$ depth-invariance on $13/14$ pretrained transformers.} A closed-form corollary of the LN-kernel analysis predicts the residual stream's smallest singular value cannot fall below its input-embedding value at any block, $\sigma_{\min}(X_\ell)/\sigma_{\min}(X_0) \ge 1$, read from a single forward pass with no gradient access. The bound holds on $13$ of $14$ pretrained transformers once each model is read on its own input distribution. The one exception, Gemma4-31B, is itself a detection win: the same read pinpoints a genuine intrinsic dead direction in its residual stream, above the fp$32$ measurement floor, at the depth where the corollary's no-cancellation hypothesis fails (\S\ref{sec:experiments}, Fig.~\ref{fig:depth_invariance}; App.~\ref{app:experiments:gemma_diagnosis}).

\item \textbf{LN/RMSNorm algebraic dichotomy as a normalisation-architecture classifier.} The same closed-form test partitions any pretrained transformer's normalisation choice: LN-equipped models exhibit the $\gamma^{-1}/\|\gamma^{-1}\|$ kernel direction; RMSNorm models do not (no $\gamma$-derived universal kernel). The dichotomy generalises to alternative norms by structural inspection (InstanceNorm, DyT, ScaleNorm, FixUp, WeightNorm classify by the same projector test; BatchNorm and GroupNorm are structurally different and remain open empirical cells; App.~\ref{app:experiments:scale}).
\end{enumerate}

  \section{Background and notation}
\label{sec:background}

For an analytic statistical family $\{p_\theta\}$ with Fisher information $\fisher(\theta)$ and $p_{\theta_0} = p^*$, the singular set $\Sigma = \{\theta : \det \fisher(\theta) = 0\}$ is where $\fisher$ loses rank. Singular learning theory \citep{Watanabe09} shows minima of $K(\theta) := \kl(p^* \| p_\theta)$ in singular models lie on $\Sigma$, and that $K$ vanishes along a unit direction $u$ at $\theta_0$ at a rate controlled by an integer KL order $k \ge 1$: $K(\theta_0 + tu) = c\,t^{2k} + O(t^{2k+1})$. In the resolved local normal form $K \sim u^{2k} + v_1^2 + \cdots + v_m^2$, the local RLCT is $\rlct = 1/(2k) + m/2$ \citep{Hironaka64, Watanabe09}; the directional invariant our rate theorem returns is the $1/(2k)$ contribution.

For a layer with weights $W_\ell$, the KFAC factorisation \citep{MartensGrosse15} gives $F_\ell \approx A_\ell \otimes G_\ell$ with $A_\ell := \expect[X_{\ell-1} X_{\ell-1}^\top]$ and $G_\ell := \expect[\delta_\ell \delta_\ell^\top]$; thus $\lambda_{\min}(F_\ell) = \lambda_{\min}(A_\ell)\,\lambda_{\min}(G_\ell)$. We use $A_\ell, G_\ell$ as canonical names. Other recurring symbols: $L$ depth, $\ell$ layer index, $h$ hidden width; $X_\ell \in \reals^{N \times h}$ activation matrix; $u$ dead direction (Def.~\ref{def:dead_direction}); $K^{\mathrm{fwd}}(\ell)$ forward shortest-weighted-path distance on a residual DAG.
 \subsection{Related work}
\label{sec:related}

\paragraph{SLT-side invariants and spectral monitoring.}
Singular learning theory \citep{Watanabe09, Watanabe18} locates the complexity of an overparameterised model in the local geometry of its loss-singular set; the local learning coefficient \citep[LLC;][]{LauFurmanWangMurfetWei25, HooglandWangFarrugiaRoberts24, WangHoogland24, Plummer25} estimates Watanabe's RLCT via SGLD sampling and provides the established Bayesian-posterior reading of singular complexity. The diagnostic this paper advances reads the same singular structure off a network's affine parameters in closed form, without sampling, and is therefore complementary to LLC's integrated Bayesian quantity. A parallel spectral-monitoring lineage tracks activation or weight spectra as generalisation-correlated signatures \citep{Jha2026NerVE, Ettori2025EigenTrack, Xu2026SpectralEdge, Truong2026SpectralEntropy, StaatsThammRosenow24, BoixAdseraLittwinAbbe23}; we sit at the bottom of the activation spectrum specifically, with an algebraic prediction for which direction $u_\bot$ is at any LN-equipped pretrained transformer. Random-matrix characterisations of the Fisher / Hessian spectrum at large width \citep{PenningtonWorah18, KarakidaAkahoAmari19, KarakidaAkahoAmari21} are used only defensively in the $n/d$ calibration of the measurement protocol. Top-of-spectrum sharpness work (Hessian-eigenvalue tracking via Lanczos/Hutchinson \citep{SagunEvciGuney17, GhorbaniKrishnanXiao19, YaoGholamiKeutzerMahoney20}, edge of stability \citep{CohenKaurLi21EdgeOfStability}) and circuit-level mech-interp phase detection \citep{OlssonElhage22InductionHeads, NandaChanLieberum23} are complementary axes that probe different objects.

\paragraph{Residual-stream + LayerNorm geometry.}
\citet{PapyanHanDonoho20} characterised neural collapse at the terminal phase of training. \citet{DongCordonnierLoukas21} show pure-attention chains lose rank doubly-exponentially with depth; \citet{NociAnagnostidisBiggio22} characterise the signal-propagation mechanism. Corollary~\ref{pred:sigma_min_res} below is the opposite conclusion in the residual case: identity skips turn the same $\sigma_{\min}(X_\ell)$ observable from a collapse into a rate-$0$ preservation along the residual stream. Concurrent dimensional-collapse work \citep{WangGeShu25} reports residual-stream activations resist rank collapse compared to attention outputs in pretrained transformers; we sharpen this to a $\sigma_{\min}$ depth-invariance prediction with an analytic mechanism and per-architecture falsifications (\S\ref{sec:experiments}). PCA-based activation rotation \citep{ashkboos2024slicegpt} folds the LN mean-subtraction projector $M = I - (1/d)\mathbf{1}\mathbf{1}^\top$ into adjacent weights as a compression rewrite, the $\gamma \equiv 1$ case of our kernel-direction identity. We generalise to the $\gamma$-weighted direction $\gamma^{-1}/\|\gamma^{-1}\|$ (Proposition~\ref{pred:ln_kernel}), which is what the bottom singular direction of the post-final-norm centred covariance actually tracks on trained models with $\gamma$-dispersion (DINOv2-base, DINOv3-ViT-L). SliceGPT's RMSNorm rotational equivariance (exploited there for compression) is the same property our Proposition~\ref{pred:rmsnorm_no_kernel} reads as a structural classifier: rotational equivariance means no preferred direction is dead a priori. We use this dichotomy as a \emph{diagnostic}, identified algebraically in the LN case and structurally in the residual-DAG case (Corollary~\ref{pred:sigma_min_res}), with no weight modification. \citet{riechers2024ln} characterises the geometry and dynamics of LayerNorm as a forward-map operator; the forward-pass-free $\gamma^{-1}/\|\gamma^{-1}\|$ kernel-direction identity (Proposition~\ref{pred:ln_kernel}) and the corresponding negative differential for RMSNorm (Proposition~\ref{pred:rmsnorm_no_kernel}) complement that forward-map analysis with a kernel-direction identity readable from the affine alone.

\paragraph{Empirical Fisher and KFAC.}
\citet{KunstnerHennigBalles19} cautioned about the empirical Fisher; our measurement protocol uses an expected-Fisher KFAC approximation \citep{MartensGrosse15, GrosseMartens16}, with \citet{EschenhagenImmerTurner23} as the transformer-generalisation reference and EKFAC \citep{GeorgeLaurentBouthillier18} as the eigenbasis variant. We use the KFAC factorisation here as a measurement primitive only; the analytical use of KFAC for deriving per-layer singular collapse rates is established in \theorycitep{}.

\section{The LN-kernel direction}
\label{sec:predictions}

Pretrained transformers carry a closed-form dead direction that a practitioner can read off a single architectural parameter. The four results we state below stake out this claim. Two are algebraic identities of the LayerNorm and RMSNorm forward maps; the third lifts those identities to a depth-invariance statement about residual-stream activations; the fourth turns the algebraic kernel direction into a quantitative per-checkpoint diagnostic. Full proofs are in \theorysrc, where these are Proposition~63 (the LayerNorm and RMSNorm kernels), Corollary~58 ($\sigma_{\min}$ depth-invariance), and Lemma~18 (the Schur ratio).

\begin{definition}[Dead direction]
\label{def:dead_direction}
A unit direction $u \in \reals^d$ is a \emph{dead direction} at $\theta_0$ if $u^\top \fisher(\theta(t)) u \to 0$ as $t \to 0$ along $\theta(t) := \theta_0 + tu$. The KL order along $u$ is the integer $k \ge 1$ with $K(\theta(t)) = c\,t^{2k} + O(t^{2k+1})$, $c > 0$. (Unrelated to the activation-level ``dead ReLU'' phenomenon despite the lexical overlap.)
\end{definition}

The starting observation is that LayerNorm's mean-subtraction projector forces a deterministic kernel direction at the output of every LN block, for any input distribution. The direction is the inverse of the LN scale parameter, normalised. This is a property of LN's forward map alone: any LN-equipped network carries it the moment its weights are initialised, before any training has occurred.

\begin{proposition}[LayerNorm kernel direction]
\label{pred:ln_kernel}
Let $\mathrm{LN}(x) = \gamma \odot \frac{P x}{\|P x\|/\sqrt d} + \beta$ with mean-subtraction projector $P = I - (1/d) \mathbf{1}_d \mathbf{1}_d^\top$, and let $X \in \reals^d$ be any random input with $P X \ne 0$ almost surely. Write $C := \cov(\mathrm{LN}(X))$. If every $\gamma_i \ne 0$, then $C \cdot \gamma^{-1} = 0$ where $\gamma^{-1} := (\gamma_1^{-1}, \ldots, \gamma_d^{-1})^\top$, with unit kernel direction
$
v^* = \gamma^{-1}/\|\gamma^{-1}\|.
$
The direction is readable from the LN affine alone, with no forward pass. (If $\gamma$ has zero coordinates $Z = \{i : \gamma_i = 0\}$, $C \cdot e_i = 0$ for every $i \in Z$.)
\end{proposition}

\begin{proof}[Sketch]
The mean-subtraction projector annihilates $\mathbf{1}_d$ ($\mathbf{1}_d^\top P = 0$), so $\mathbf{1}_d^\top \tilde x_{\mathrm{LN}}(X) = 0$ almost surely and $\cov(\tilde x_{\mathrm{LN}}) \mathbf{1}_d = 0$. With $\mathrm{LN}(X) = \mathrm{diag}(\gamma) \tilde x_{\mathrm{LN}}(X) + \beta$, $C = \mathrm{diag}(\gamma) \cov(\tilde x_{\mathrm{LN}}) \mathrm{diag}(\gamma)$. Since $\mathrm{diag}(\gamma) \gamma^{-1} = \mathbf{1}_d$ coordinate-wise, $C \gamma^{-1} = \mathrm{diag}(\gamma) \cov(\tilde x_{\mathrm{LN}}) \mathbf{1}_d = 0$.
\end{proof}

The proof is short because LayerNorm does the work: the projector $P$ algebraically kills $\mathbf{1}_d$, and the subsequent $\gamma$-rescaling carries that kernel to $\gamma^{-1}/\|\gamma^{-1}\|$. The contribution is not the algebra but the empirical claim it forces. Any pretrained LN-equipped transformer should exhibit a bottom singular direction that aligns with $\gamma^{-1}/\|\gamma^{-1}\|$ at random initialisation; deviation falsifies the prediction.

The complementary case is RMSNorm. Lacking a mean-subtraction projector, RMSNorm has no operator that annihilates a fixed direction independent of input. Every $\gamma_i \ne 0$ leaves the RMSNorm output covariance full rank for any non-degenerate input distribution. We state this as a negative differential to make the dichotomy formal.

\begin{proposition}[RMSNorm has no universal kernel]
\label{pred:rmsnorm_no_kernel}
For an RMSNorm operator $\mathrm{RMSNorm}(x) = \gamma \odot x / \sqrt{\|x\|^2/d}$ with every $\gamma_i \ne 0$, no unit direction $v(\gamma)$ depending only on $\gamma$ universally satisfies
$
v^\top \cov(\mathrm{RMSNorm}(X)) v = 0
$
across all input distributions $X$ with $\cov(X) \succ 0$.
\end{proposition}

\begin{proof}[Sketch]
$\mathrm{RMSNorm}(X) = r(X) \cdot \gamma \odot X$ with $r(X) > 0$ a positive scalar; $\mathrm{RMSNorm}$ has no projector that annihilates a fixed direction independent of $X$. For $X \sim \mathcal{N}(0, I_d)$, $r(X) X = \sqrt d \cdot X / \|X\|$ is uniform on $S^{d-1}$ scaled by $\sqrt d$, so $\cov(r(X) X) = I_d$. Hence $\cov(\mathrm{RMSNorm}(X)) = \mathrm{diag}(\gamma)^2 \succ 0$, with no zero direction. A universal kernel direction $v(\gamma)$ would require a zero direction for every input distribution; this counter-example rules it out.
\end{proof}

Together the two propositions stake an architectural dichotomy: a single closed-form test on the affine parameter classifies any pretrained transformer's normalisation choice, before any forward pass. LN-equipped models exhibit the predicted direction; RMSNorm models do not.

The next result lifts the LN-kernel observation to a depth-invariance claim that holds under either normalisation. Residual streams in modern transformers compose blocks via additive-identity skips, $X_{\ell+1} = X_\ell + F(X_\ell)$. The skip provides a zero-cost backward path for the bottom singular direction, so the smallest singular value cannot decrease block-to-block under a non-cancellation hypothesis on the residual branch. This is independent of whether the block uses LN or RMSNorm.

\begin{corollary}[Residual-stream $\sigma_{\min}$ depth-invariance]
\label{pred:sigma_min_res}
On a residual DAG with exact-identity skips, the smallest singular value of the residual-stream activation matrix satisfies $\sigma_{\min}(X_\ell) / \sigma_{\min}(X_0) \ge 1$ at every block under canonical alignment of dead directions and the no-cancellation hypothesis at the bottom-of-active-spectrum direction. (The hypothesis is not generic on trained networks; we report it as a checkpoint-failable architectural property.)
\end{corollary}

\begin{proof}[Sketch]
Path decomposition: $X_\ell = X_0 + \sum_{p \in \mathcal{P}(\ell)} \delta_p$ where $\mathcal{P}(\ell)$ enumerates non-skip paths reaching block $\ell$. Under no-cancellation at the bottom-of-active-spectrum direction, the perturbations $\delta_p$ do not destructively interfere with $X_0$ along that direction, so the sum's bottom singular value is bounded below by the input's. Cancellation breaks the bound; we observe this on Gemma~$4$ (Section~\ref{sec:experiments}, App.~\ref{app:experiments:gemma_diagnosis}).
\end{proof}

The fourth result sharpens the LN-kernel test from a binary check on direction alignment into a quantitative diagnostic. Proposition~\ref{pred:ln_kernel} predicts that the bottom singular direction $u_\bot$ of the post-final-norm centred covariance equals $\gamma^{-1}/\|\gamma^{-1}\|$. What this leaves out is how \emph{deeply} the kernel has formed: at random initialisation the predicted direction is the kernel by construction, but on a trained network the actual smallest eigenvalue can be orders of magnitude smaller along some \emph{other} direction that training has driven toward zero. The Schur ratio captures this depth.

\begin{lemma}[Schur-ratio prefactor]
\label{pred:schur_constant_quantitative}
For a LayerNorm-fed Linear layer with Gaussian-isotropic input at the post-final-norm position, the per-block Schur ratio
$
R(h) := \lambda_{\min}(A_\ell)\,/\,(A_\ell)_{u, u}
$
along the predicted direction $u = \gamma^{-1}/\|\gamma^{-1}\|$ converges to $1$ at the asymptotic limit. Equivalently, the dead-channel depth $\Delta_\gamma := 1 - R(h) \in [0, 1]$ tracks the finite-$N$ Marchenko--Pastur dispersion $\sim 2\sqrt{h/N} - h/N$ at random initialisation, and saturates at $\approx 1$ when training opens additional dead directions strictly deeper than $u$.
\end{lemma}

\begin{proof}[Sketch]
Under Gaussian-isotropic post-LN input, the activation Gram $A_\ell = \mathbb{E}[X X^\top]$ at the dead row $u$ has the closed-form structure $\alpha I + \beta \mathbf{1}\mathbf{1}^\top$ on the orthogonal complement of $u$. The Schur complement at the dead row gives a rank-$1$ outer-product cancellation $c = (h-1)/(\pi(\pi + h - 2)) \le 1/\pi$, leaving $R(h) = 1 - c \ge 1 - 1/\pi$ at finite $h$ and $R(h) \to 1$ as $h \to \infty$. The Marchenko--Pastur term is the bottom-edge dispersion of the sample-covariance estimator at finite $N$.
\end{proof}

The four results compose into a single empirical pipeline. Proposition~\ref{pred:ln_kernel} gives the LN architectural baseline, validated at four-decimal precision on $9$ LN models. Proposition~\ref{pred:rmsnorm_no_kernel} is the dichotomy's negative, validated on $5$ RMSNorm models. Corollary~\ref{pred:sigma_min_res} is the residual-DAG consequence, tested on all $14$ models in our set. Lemma~\ref{pred:schur_constant_quantitative} converts the LN-kernel test into a per-checkpoint accumulation budget on the Pythia size sweep. Section~\ref{sec:experiments} carries each test in turn.
  
\section{Experiments}
\label{sec:experiments}

We test the four predictions of \S\ref{sec:predictions} on a cross-architecture set of $14$ pretrained transformers ($160$M--$35$B parameters; objectives span next-token CE, supervised image classification, image SSL, image--text contrastive, and video SSL). Each prediction targets a different aspect of the dead-direction structure. Corollary~\ref{pred:sigma_min_res} stakes a residual-DAG architectural claim that should hold across all $14$ models irrespective of normalisation: the smallest residual-stream singular value cannot decrease block-to-block under the no-cancellation hypothesis. Propositions~\ref{pred:ln_kernel} and~\ref{pred:rmsnorm_no_kernel} together stake an architectural dichotomy: LN admits a parameter-readable kernel direction $\gamma^{-1}/\|\gamma^{-1}\|$, RMSNorm does not. Lemma~\ref{pred:schur_constant_quantitative} sharpens the LN-kernel test from a binary check on the measured bottom direction $u_\bot$ to a quantitative per-checkpoint diagnostic via the Schur ratio $\Delta_\gamma := 1 - R(h)$. Each test runs at the cost of a single forward pass per checkpoint with no SGLD chain and no gradient access; we validate each in turn below.

\paragraph{(a) Pipeline and protocol.} The pipeline captures residual-stream activations $X_\ell \in \reals^{N \times h}$ per block, accumulates a chunked covariance (fp$32$ default; fp$64$ recommended for inputs with $\sigma_{\max} \gtrsim 10^5$ to anchor the bottom of the spectrum), then fp$64$-eigendecomposes once per layer to recover $\sigma_{\min}$, $\sigma_{\max}$, top/bottom singular directions, and Shannon-entropy effective rank \citep{RoyVetterli07,SanyalEtAl20}. Per-checkpoint $\sigma_{\min}$ alone is real-time-cheap on a single RTX 3090 ($\le 2$\,s through Qwen3.6-27B at fp$32$ cov, ${\sim}3$--$5\times$ slower at fp$64$); the full residual-stream fingerprint takes $<2$\,s at Pythia-1B, $8$\,s at Qwen3.5-4B, $160$\,s at Qwen3.6-27B at fp$32$ cov (Appendix~\ref{app:experiments:scale}). All three predictions are read from a single forward pass per checkpoint, with no SGLD chain and no gradient access. Calibration corpus: WikiText-103 validation for language models, ImageNet validation for vision transformers; observable values are input-distribution-invariant under the centred-covariance protocol, confirmed empirically across $5$ standard ViTs at $|\Delta\cos| < 10^{-6}$ between synthetic-noise and real-image calibration.

\paragraph{(b) \texorpdfstring{$\sigma_{\min}$}{sigmamin} depth-invariance.} Corollary~\ref{pred:sigma_min_res} (\S\ref{sec:predictions}) holds on $12/14$ pretrained transformers under the uniform text-only protocol (fp$64$ centred-cov, against the true input embedding $X_0$). The two failures are both Gemma~$4$ releases: Gemma4-31B (non-monotone depth-decreasing, a finite $\sim\!18\times$ net $\sigma_{\min}$ reduction at the output) and Gemma4-E4B (an intermediate dip that recovers). A family discriminator across Gemma~$1$--$4$ localises the failure to the Gemma~$4$ release: pre-LN Gemma-$1$ and Peri-LN Gemma-$2$/$3$ all preserve $\sigma_{\min}$, and only Gemma~$4$ (also Peri-LN) fails, so the failure does not track the normalisation placement. A five-mechanism ablation rules out four further candidate causes and localises the collapse to the MLP sub-layer; the residual cause is a calibration-distribution effect of Gemma~$4$'s encoder-free native multimodality (App.~\ref{app:experiments:gemma_diagnosis}). Text-only calibration leaves the vision and audio soft-token subspaces dormant: under the full input distribution the dip dissolves on the tri-modal Gemma4-E4B (a distribution-relative dead direction), while a residual intrinsic near-dead direction persists, above the fp$32$ measurement floor, on the bi-modal Gemma4-31B. So $13/14$ hold once each model is calibrated on its own input distribution. The prediction also holds dynamically: on $8$ Pythia-1B pretraining revisions (\texttt{step1} to \texttt{step143000}), the ratio $\ge 1$ at every depth on every revision (Fig.~\ref{fig:pythia_revisions_arc}b), amplitude rising from $\sim 12\times$ to a peak of $\sim 299\times$ before settling to $\sim 206\times$ at the mature checkpoint. Proposition~\ref{pred:ln_kernel}'s cosine stays $\ge 0.998$ across all $8$ revisions while the uniform-$\gamma$ approximation decays as $\gamma$ disperses (Fig.~\ref{fig:pythia_revisions_arc}d): the prediction's $\gamma$-dependence is what tracks the dispersion. $\sigma_{\min}$ profiling at fp$64$ cov is the cheapest signal we know that surfaces this kind of residual-branch pathology in pretrained transformers.

\begin{figure}[t]
\centering
\includegraphics[width=0.95\textwidth]{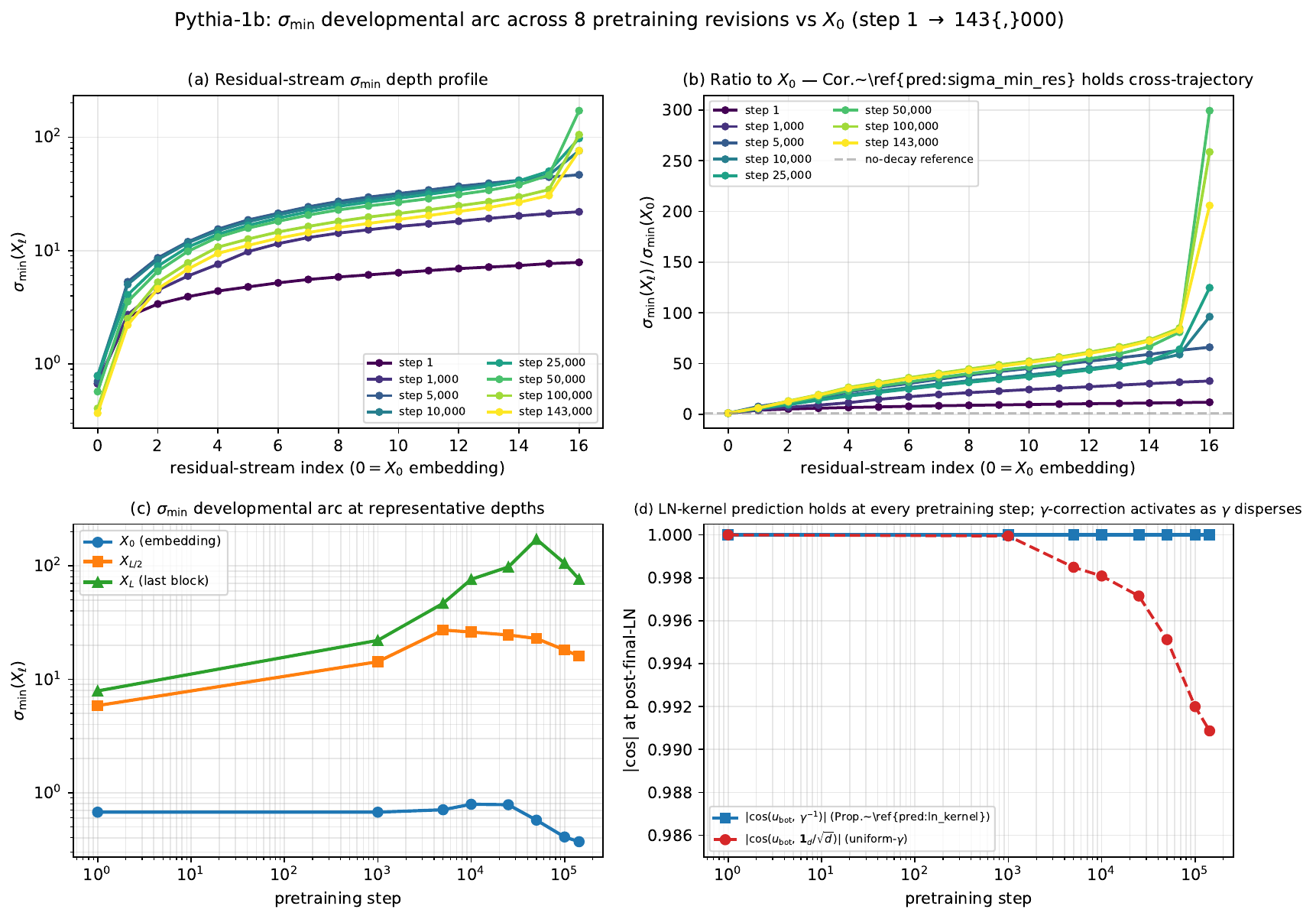}
\caption{Pythia-1B $\sigma_{\min}$ developmental arc across $8$ pretraining revisions (step $1$ $\to$ $143{,}000$). \textbf{(a)}~Residual-stream $\sigma_{\min}$ depth profile, one curve per revision; the bottom curve is \texttt{step1} (initialisation), the top curves are mature checkpoints. \textbf{(b)}~Ratio $\sigma_{\min}(X_\ell)/\sigma_{\min}(X_0)$ exceeds $1$ at every depth on every revision, so Corollary~\ref{pred:sigma_min_res} holds across pretraining time, not only at the mature checkpoint. The ratio amplitude rises from $\sim 12\times$ at \texttt{step1} to a peak of $\sim 299\times$ at \texttt{step50000}, settling to $\sim 206\times$ at the mature checkpoint. \textbf{(c)}~$\sigma_{\min}$ at three representative depths vs pretraining step. \textbf{(d)}~LN-kernel alignment at the post-final-norm position vs pretraining step: $|\cos(u_\bot, \gamma^{-1}/\|\gamma^{-1}\|)| \ge 0.998$ on $8/8$ revisions (Proposition~\ref{pred:ln_kernel} validated at every checkpoint), while the uniform-$\gamma$ approximation $|\cos(u_\bot, \mathbf{1}_d/\sqrt{d})|$ decays from $1.000$ at \texttt{step1} to $0.991$ at \texttt{main} as $\gamma$ trains and disperses. The gap between the two curves in (d) is the $\gamma$-dispersion accumulating during pretraining; the prediction's $\gamma$-dependence is what tracks it.}
\label{fig:pythia_revisions_arc}
\end{figure}
 
\paragraph{(c) Norm-side algebraic dichotomy.} For each of the $9$ LN-equipped pretrained transformers we read $\gamma^{-1}/\|\gamma^{-1}\|$ off the terminal-norm affine and compare it to the bottom singular direction $u_\bot$ of the post-final-norm centred covariance, computed via fp$64$ SVD on a calibration batch. Proposition~\ref{pred:ln_kernel} predicts the two should agree exactly. Empirically $|\cos(u_\bot, \gamma^{-1}/\|\gamma^{-1}\|)| \ge 0.9999$ on $9/9$ models, mean $0.99999$. Four-decimal alignment is the load-bearing test: the random-direction baseline at typical residual widths is $|\cos| \approx 1/\sqrt{d} \le 0.05$, so the empirical match sits $37\!\times \pm 9\!\times$ above the null and would be unreachable by any non-architectural mechanism. The complementary RMSNorm test evaluates Proposition~\ref{pred:rmsnorm_no_kernel}: across $5$ RMSNorm models and $n{=}718$ (model, position) cells, $|\cos(u_\bot, e_{\arg\min \gamma})|$ has mean $0.015$, indistinguishable from the random null. The two readings together classify any pretrained transformer's normalisation choice from a closed-form inspection of its forward-map projector, before any forward pass on the checkpoint.

\begin{table}[t]
\centering\footnotesize
\caption{Schur-ratio $\Delta_\gamma(h) := 1 - R(h)$ at the QKV input along the predicted direction $u = \gamma^{-1}/\|\gamma^{-1}\|$, $11$ pretrained transformers, $N = 131{,}072$ tokens fp$64$ covariance. Pythia LN rows: 5-seed random-init pooled median across all attention sites (with IQR; noise-floor recovery applied where both $\lambda_{\min}(A_\ell)$ and $(A_\ell)_{u,u}$ sit at the fp$64$ floor, see App.~\ref{app:experiments:schur}). Trained Pythia checkpoints saturate at $\Delta_\gamma \approx 1$. RMSNorm has no universal algebraic kernel (Prop.~\ref{pred:rmsnorm_no_kernel}); single-seed randinit baseline reported.}
\label{tab:schur_pole_c}
\setlength{\tabcolsep}{6pt}
\begin{tabular}{@{}l c r c c@{}}
\toprule
Model & Norm & $h$ & $\Delta_\gamma$ random (pooled median [IQR]) & $\Delta_\gamma$ trained \\
\midrule
Pythia-$70$M    & LN              &  $512$ & $\mathbf{+0.232}$ \,[+0.08, +0.30] & $0.999$ \\
Pythia-$160$M   & LN              &  $768$ & $\mathbf{-0.010}$ \,[-0.03, -0.01] & $0.999$ \\
Pythia-$410$M   & LN              & $1024$ & $\mathbf{-0.025}$ \,[-0.05, -0.01] & $1.000$ \\
Pythia-$1$B     & LN              & $2048$ & $\mathbf{-0.015}$ \,[-0.03, -0.01] & $1.000$ \\
Pythia-$2.8$B   & LN              & $2560$ & $\mathbf{-0.016}$ \,[-0.03, -0.01] & $1.000$ \\
\midrule
Qwen$2.5$-$0.5$B & RMSN pre-LN    &  $896$ & $0.962$ & $0.998$ \\
LLaMA-$2$-$7$B   & RMSN pre-LN    & $4096$ & $0.876$ & $1.000$ \\
Gemma-$2$-$2$B   & RMSN sandwich  & $2304$ & $0.957$ & $0.988$ \\
Gemma-$3$-$1$B   & RMSN sandwich  & $1152$ & $0.945$ & $1.000$ \\
Gemma-$3$-$4$B   & RMSN sandwich  & $2560$ & $0.970$ & $1.000$ \\
Gemma-$4$-E$4$B  & RMSN sandwich  & $2560$ & $0.803$ & $0.999$ \\
\bottomrule
\end{tabular}
\end{table}
 
\paragraph{(d) Schur-ratio diagnostic: algebraic null vs trained singular structure.}
\label{sec:exp:scale:schur}
The Schur-ratio prefactor (Lem.~\ref{pred:schur_constant_quantitative}) yields a width-only invariant $R(h) := \lambda_{\min}(A_\ell)/(A_\ell)_{u,u}$ along the predicted direction $u = \gamma^{-1}/\|\gamma^{-1}\|$, evaluated at the QKV input of every attention block (a \emph{Type-A site}: a Linear layer fed by a LayerNorm). We sweep $11$ pretrained transformers ($5$ LayerNorm Pythia sizes plus $6$ RMSNorm models) at $N = 131{,}072$ tokens fp$64$ covariance, each model measured at the trained checkpoint and at random init across $5$ initialization seeds (Tab.~\ref{tab:schur_pole_c}). At random init on Pythia~$160$M--$2.8$B, the pooled-median $\Delta_\gamma(h) := 1 - R(h)$ across $5$ seeds and all attention sites holds within $\pm 0.03$ of zero, consistent with the prediction's algebraic limit $R \to 1$ up to fp$64$ covariance noise; Pythia-$70$M ($h{=}512$, $6$ blocks) reads pooled-median $\Delta_\gamma \approx 0.23$, the small-block, wider-eigengap edge of the sweep. On the trained checkpoint $\Delta_\gamma \to 1$ uniformly across all $5$ Pythia sizes, indicating that training opens dead directions $\sim\!10^3\times$ deeper along some other coordinate (Pythia-$70$M block~$0$: $\lambda_{\min}(A) = 1.96 \cdot 10^{-3}$ vs $(A)_{\gamma^{-1}, \gamma^{-1}} = 2.40$; ratio $1.2 \cdot 10^3$). RMSNorm models start at $\Delta_\gamma \in [0.80, 0.97]$ at random init (no universal kernel exists, so the algebraic baseline is absent) and saturate at $\approx 1$ once trained; Lemma~\ref{pred:schur_constant_quantitative}'s asymptotic limit holds for LN-fed Type-A sites and not for RMSNorm-fed ones, as predicted. Full per-row analysis, including the cross-model figure (Fig.~\ref{fig:schur_pole_c}), the Pythia width sweep, and the noise-floor recovery protocol for blocks where both $\lambda_{\min}(A)$ and $(A)_{u,u}$ sit at the fp$64$ cov floor, is in App.~\ref{app:experiments:schur}.
  \section{Discussion}
\label{sec:discussion}

\paragraph{What we showed.} The framework's central architecture-side prediction ($\gamma^{-1}/\|\gamma^{-1}\|$ as the post-final-LN dead direction, readable from the affine parameter alone) holds at four-decimal precision on every LayerNorm pretrained transformer in our set ($9/9$ centered) and is correctly absent on RMSNorm models ($5/5$, $|\cos| \le 0.10$ centered). The same algebraic null acts as a per-checkpoint reference: on the trained checkpoint, the Schur ratio $\Delta_\gamma$ along the predicted direction collapses from the finite-$N$ Marchenko--Pastur drift on the smallest cell ($\approx 0.37$ at $h{=}512$, within $\pm 0.06$ of zero at $h \ge 768$ on the Pythia sweep) to $\approx 1$ across all sizes on the trained checkpoint, so training has opened additional dead directions $\sim\!10^3\times$ deeper along the same coordinate. The gap between random-init and trained $\Delta_\gamma$ is itself a forward-pass-only diagnostic of accumulated singular structure that no prior method extracts. Two consequences follow without further machinery: residual-stream $\sigma_{\min}$ depth-invariance on $13/14$ pretrained transformers once each model is read on its own input distribution (both Gemma~$4$ releases dip under the uniform text-only protocol; multimodal recalibration dissolves Gemma4-E4B's dip as a dormant-modality artifact and leaves Gemma4-31B as the one genuine intrinsic dead direction the same read pinpoints), and an algebraic LN/RMSNorm dichotomy that classifies any pretrained transformer's normalisation choice from a single closed-form test on $\gamma^{-1}/\|\gamma^{-1}\|$. The fingerprint reads only the activation covariance and the LN affine, so the estimator question that scopes the trajectory rate-fit does not reach it: $\sigma_{\min}$ depth-invariance is a singular-value of the forward-pass covariance with no loss gradient and no Fisher estimator involved, and $\gamma^{-1}/\|\gamma^{-1}\|$ is read from the affine parameter alone (Remark~4). The residual prefactor that distinguishes the stored loss-gradient covariance from the population Fisher on a learned trajectory enters neither.

\paragraph{Alternative explanations considered.} \emph{(i) Projector-kernel coincidence.} If $\gamma$ were uniform, $\gamma^{-1}/\|\gamma^{-1}\|$ would collapse to $\mathbf{1}_d/\sqrt{d}$ (LN's mean-subtraction projector kernel) and the prediction would be indistinguishable from the trivial null. Three of the nine LN models break that regime materially: DINOv$2$-base has $|\cos(\gamma^{-1}/\|\gamma^{-1}\|, \mathbf{1}_d/\sqrt{d})| = 0.56$ from the affine, DINOv$3$-ViT-L has $0.045$, and ViT-B/16 has $0.007$ (the most extreme, its $\gamma$ nearly orthogonal to $\mathbf{1}_d$). On all three, the measured $u_\bot$ tracks the exact $\gamma^{-1}$ prediction to $\ge 3$ decimals while the uniform-$\gamma$ approximation tracks the smaller cosine the affine actually predicts; the $\gamma$-dependence carries the prediction. \emph{(ii) fp$32$ tautology.} At fp$32$ cov, several models read $\sigma_{\min}(X_0) = 0$ from the accumulator floor; only fp$64$ resolves the actual values, and only then is the strict ratio test passable or failable. For the two massive-activation RMSNorm models the binding floor is the bf$16$ forward capture as well as the fp$32$ covariance: fp$32$ forward and fp$64$ cov are both required to resolve them. The Gemma~$4$ failures persist at fp$64$.

\paragraph{Limitations.} Two are load-bearing. \emph{(a) The origin of Gemma4-31B's intrinsic dead direction is open.} The $\sigma_{\min}$ collapse is dissociated from five candidate mechanisms ($\gamma$-kernel formation, residual-vs-delta alignment, the Peri-LN sandwich-norm pattern, sink-token estimator bias, post-pretraining stage), localised to the MLP sub-layer, and on the tri-modal Gemma4-E4B traced to a calibration-distribution effect: the text-only dip is a dormant-modality artifact that dissolves under the model's full text$+$image$+$audio distribution. On the bi-modal Gemma4-31B a genuine intrinsic near-dead direction persists above the fp$32$ measurement floor under image calibration; the training-recipe difference from the earlier Gemma releases that would explain why this checkpoint carries one remains untested. \emph{(b) Architecture coverage is LN-vs-RMSNorm only.} BatchNorm (CNN backbones), GroupNorm (diffusion U-Nets), and DyT are structurally different and remain open empirical cells; the algebraic test extends to them by inspection (Proposition~\ref{pred:ln_kernel}'s mean-subtraction projector is the structural property that determines whether a $\gamma$-derived universal kernel direction exists), but per-norm validation on real architectures is left to future work.

\paragraph{Falsifier.} The central prediction has a single clean falsifier: an LN-equipped pretrained transformer for which $|\cos(u_\bot, \gamma^{-1}/\|\gamma^{-1}\|)|$ at the post-final-norm position is bounded away from $1$ on the trained checkpoint at fp$64$ centered cov. Across the $14$-model set we did not find one; any such counter-example would force a structural revision of Proposition~\ref{pred:ln_kernel}.

\paragraph{Implications: architectural vs training-induced singular structure.} The split between random-init and trained $\Delta_\gamma$ operationalises a worldview about how singular structure accumulates in transformers. Every LN block carries one architecturally guaranteed dead direction, $\gamma^{-1}/\|\gamma^{-1}\|$, present at random init and predictable from the affine alone. Training opens deeper kernels along the same coordinate, and the gap between architectural baseline and trained depth is the training-induced component. RMSNorm models have no architectural component; their dead structure is entirely training-induced, and their $\sigma_{\min}$ depth-invariance pass rate is comparable to LN's at convergence. The LN/RMSNorm trade is therefore between capacity (RMSNorm preserves all $d$ residual-stream coordinates) and structural predictability (LN gives a closed-form, parameter-readable anchor). Modern LLMs weighted capacity above predictability when they switched to RMSNorm. The predictability remains valuable wherever LN is still in use, both for diagnostics on existing checkpoints and for analysis of pretrained vision transformers.

\paragraph{Practitioner takeaways.} Four concrete actions become available. \emph{(1) Free architectural reference and protocol sanity check.} On any LN-equipped checkpoint, $\gamma^{-1}/\|\gamma^{-1}\|$ is the bottom singular direction of the post-final-norm centred covariance at random init by algebraic identity, with no compute required. The same direction doubles as a measurement-protocol check: if the empirical $u_\bot$ on a calibration batch does not match $\gamma^{-1}/\|\gamma^{-1}\|$ within fp$64$ precision, the protocol is wrong (incorrect hook, raw-Gram instead of centred-cov, fp$32$ accumulator floor). \emph{(2) Importance-based LoRA placement.} For methods that select adapter directions by importance score (AdaLoRA, gradient-importance LoRA), $\gamma^{-1}$ at LN block boundaries should be excluded from the candidate set. Rank put there cannot adapt to input variation because the LN output has zero variance along $\gamma^{-1}$; it can only learn a bias term. Random-initialised LoRA overlaps $\gamma^{-1}$ at the $1/\sqrt{d}$ level and is mostly unaffected, but importance-based methods that align with high-variance directions can waste rank if they fail to exclude this one. \emph{(3) Investigation signal on residual-branch pathologies.} An fp$64$-cov $\sigma_{\min}$ pass surfaces architectural anomalies (the Gemma~$4$ MLP-localised collapse) that standard fp$32$-cov pipelines miss because the noise floor $\sigma_{\max}\cdot\sqrt{\varepsilon_{\mathrm{fp32}}}$ sits above the active spectrum. The diagnostic flags an anomaly worth investigating, not a verdict on functional pathology; without matched-recipe controls we cannot tell whether the anomaly is a defect or a deliberate trade. The value is targeting expensive downstream investigation rather than running it blind. \emph{(4) Norm-architecture screening for designers and compression researchers.} The algebraic dichotomy classifies any normalisation scheme by inspection of its forward-map projector, without retraining. For architecture designers introducing new schemes (DyT, ScaleNorm, FixUp, sandwich-norm variants), it predicts whether the scheme carries a $\gamma$-derived universal kernel before any empirical study. For compression researchers, it predicts which models admit SliceGPT-style affine-folding into adjacent weights.

\paragraph{Future work.} \emph{Continual-learning safe-direction constraints}: orthogonal-gradient and null-space methods identify these directions empirically; the LN-kernel proposition gives an analytic, pretrained-model-readable instance. \emph{Rank-collapse early warning during pretraining}: the static observables are computable per checkpoint with no gradient access, natural candidates for periodic monitoring at scale (Fig.~\ref{fig:pythia_revisions_arc} is the existence proof). \emph{Norm-cell extensions}: BatchNorm, GroupNorm, and DyT are forward-pass-cheap tests that extend the dichotomy beyond LN/RMSNorm.
  
\bibliographystyle{abbrvnat}

\clearpage
\appendix

\section{Experimental details and reproducibility}
\label{app:experiments_submission}

This appendix documents per-experiment hyperparameters, extended tables, and reproducibility notes for each empirical claim, organised as follows.

\textbf{Theory and scope (\S\ref{app:experiments:theory_scope_map}--\S\ref{app:experiments:predictions_extended}).} The theory-to-scope map gives per-prediction assumption sets, validation loci, and explicit failure conditions; the \emph{key-results sketch} (\S\ref{app:experiments:proof_sketches}) restates the four framework results used in the body in compact equation-level form with $1$--$3$ line proof sketches (LN-kernel direction, RMSNorm differential, $\sigma_{\min}$ depth-invariance, Schur-ratio prefactor), so a reader can verify the empirical claims against an explicit derivation without opening \theorypdf; framing-only predictions (composition additivity, quotient Fisher rate) are stated for reference.

\textbf{Measurement protocols (\S\ref{app:experiments:observable_protocols}).} Per-observable computation recipes for the activation-side measurements this paper uses: residual-stream $\sigma_{\min}$ via fp$64$ centred-cov SVD, the Schur-ratio $R(h)$ probe along the framework-predicted direction $u = \gamma^{-1}/\|\gamma^{-1}\|$, chunked-covariance accumulation at LLM widths, fp$64$ upcast for the massive-activations regime ($\sigma_{\max} \gtrsim 10^5$), and $n/h$ sample-budget gates.

\textbf{LLM-scale fingerprint and diagnosis (\S\ref{app:experiments:scale}, \S\ref{app:experiments:gemma_diagnosis}).} Supports the body's LLM-scale results: per-model residual-stream fingerprint table on the $14$-model set; Pythia-1B pretraining-revision developmental arc; ViT FFN fine-tuning trajectory test; the five-mechanism Gemma~$4$ anomaly diagnosis.

\textbf{Cross-architecture feasibility and compute (\S\ref{app:experiments:cross_arch}--\S\ref{app:experiments:compute}).} Tractability across the architecture set, hardware envelope, and reproducibility notes.

\subsection{Theory-to-scope map}
\label{app:experiments:theory_scope_map}

The body's §\ref{sec:experiments} opening paragraph names every load-bearing prediction with its class and validation locus. The reach-tier roll-up below is the same information at one level of abstraction higher: each row is a class of prediction with a single representative theorem; ``Off-the-shelf testable'' indicates whether the prediction can be validated on a frozen pretrained transformer with no gradient access.

\begin{center}\footnotesize
\setlength{\tabcolsep}{4pt}
\begin{tabular}{@{}l p{3.6cm} p{4.2cm} c@{}}
\toprule
Reach tier & Representative prediction & Where validated & Off-the-shelf? \\
\midrule
Architecture-agnostic structural & $\sigma_{\min}$ depth-invariance (Cor.~\ref{pred:sigma_min_res}) on residual DAGs with exact-identity skips & $14$ pretrained transformers under fp$64$ centred-cov measurement (\S\ref{sec:experiments}); $13/14$ pass on each model's own input distribution, the one exception Gemma4-31B (intrinsic near-dead direction); Gemma4-E4B's text-only dip resolves as distribution-relative & \textbf{Yes} \\
\addlinespace[1pt]
Norm-side algebraic & LN-kernel $\gamma^{-1}/\|\gamma^{-1}\|$ (Prop.~\ref{pred:ln_kernel}); RMSNorm differential (Prop.~\ref{pred:rmsnorm_no_kernel}) & $14$ pretrained transformers (\S\ref{sec:experiments}); Schur-ratio cross-model sweep & \textbf{Yes} \\
\bottomrule
\end{tabular}
\end{center}

\noindent\footnotesize Both reach tiers tested in this paper are off-the-shelf: every result above is read from a frozen pretrained transformer with no gradient access and no SGLD chain. The framework's wider rate-side predictions (Fisher decay $\Theta(t^{2(k-1)})$, selection rule, multi-layer KFAC bridge per-layer rate ladder $2(L-\ell)/\ell$, curvature--volume rate chain) require a controlled SGD-class trajectory and are not validated here.\normalsize

All $14$ pretrained transformers in our set use exact-identity residual streams (no gated or highway residuals); the MoE in Qwen3.6-35B-A3B applies routing only to the FFN sub-block, not to the residual skip itself, so Corollary~58's exact-identity-skip hypothesis applies.
 
\subsection{Key results from the framework, with proof sketches}
\label{app:experiments:proof_sketches}

The empirical claims rest on four results of \theorycite{}. We restate each in compact form with a 2--3 line proof sketch; full proofs (with all assumptions, supporting lemmas, and architectural extensions) live in  the theory paper. Throughout, $\mathrm{LN}(x) = \gamma \odot \tilde x_{\mathrm{LN}}(x) + \beta$ with $\tilde x_{\mathrm{LN}}(x) = \sqrt{d}\,Px / \|Px\|$ and $P := I - \mathbf{1}_d \mathbf{1}_d^\top / d$ the mean-subtraction projector; $\mathrm{RMSNorm}(x) = \gamma \odot x / \sqrt{\|x\|^2/d}$ has no projector. $X_\ell \in \mathbb{R}^{N \times h}$ stacks per-token residual-stream activations at block $\ell$; $A_\ell := \mathbb{E}[X_{\ell-1} X_{\ell-1}^\top]$ is the input covariance at layer $\ell$.

\paragraph{P1. LN-kernel direction (Prop.~63, \theorytag).}
For any input $X$ with $PX \ne 0$ a.s.\ (i.e., not concentrated on the constant-vector line $\mathrm{span}(\mathbf{1}_d)$) and $\gamma$ with $\gamma_i \ne 0$ for all $i$:
\[
\mathrm{cov}(\mathrm{LN}(X)) \cdot \gamma^{-1} \;=\; 0, \qquad \text{kernel direction } v^\star \;=\; \gamma^{-1} / \|\gamma^{-1}\|.
\]
\emph{Sketch (3 lines).} (a) $\mathbf{1}_d^\top \tilde x_{\mathrm{LN}}(X) = \sqrt{d}\, \mathbf{1}_d^\top P X / \|PX\| = 0$ a.s.\ since $\mathbf{1}_d^\top P = 0$. (b) Hence $\mathrm{cov}(\tilde x_{\mathrm{LN}}) \mathbf{1}_d = \mathbb{E}[\tilde x_{\mathrm{LN}}(\mathbf{1}_d^\top \tilde x_{\mathrm{LN}})] = 0$. (c) With $\mathrm{LN}(X) = \mathrm{diag}(\gamma) \tilde x_{\mathrm{LN}} + \beta$ and $\beta$ contributing nothing to covariance, $\mathrm{cov}(\mathrm{LN}(X)) \gamma^{-1} = \mathrm{diag}(\gamma)\,\mathrm{cov}(\tilde x_{\mathrm{LN}})\,\mathrm{diag}(\gamma)\, \gamma^{-1} = \mathrm{diag}(\gamma)\,\mathrm{cov}(\tilde x_{\mathrm{LN}})\,\mathbf{1}_d = 0$. The unit kernel direction $v^\star$ is read from the LN affine alone, no forward pass required.

\paragraph{P2. RMSNorm differential (Prop.~63, \theorytag).}
For $\gamma$ with $\gamma_i \ne 0$ for all $i$, no unit direction $v(\gamma)$ depending only on $\gamma$ universally satisfies $v^\top \mathrm{cov}(\mathrm{RMSNorm}(X)) v = 0$ across all input distributions with $\mathrm{cov}(X) \succ 0$.
\emph{Sketch (1 line).} Take $X \sim \mathcal{N}(0, I_d)$: $r(X) X = \sqrt{d}\,X/\|X\|$ is uniform on $\sqrt{d}\,S^{d-1}$, with $\mathrm{cov} = I_d$ by spherical symmetry; hence $\|\gamma \odot v\|^2 > 0$ for any $v \ne 0$, contradicting universality. The structural distinction is the absence of the mean-subtraction projector $P$: RMSNorm's $\tilde x_{\mathrm{RMS}}$ does not live in $\mathbf{1}_d^\perp$, so no kernel direction is forced by $\gamma$ alone.

\paragraph{C3. $\sigma_{\min}$ depth-invariance on residual DAGs (Cor.~58, \theorytag).}
For a residual DAG with exact-identity skips and arbitrary residual-branch contributions $\Delta_\ell$, $X_{\ell+1} = X_\ell + \Delta_\ell(X_\ell)$. Then under generic non-cancellation between $X_\ell$ and $\Delta_\ell$:
\[
\sigma_{(r_0)}(X_{\ell+1}) \;\ge\; \sigma_{(r_0)}(X_\ell), \qquad \text{equivalently } \sigma_{(r_0)}(X_\ell)/\sigma_{(r_0)}(X_0) \ge 1 \text{ for all } \ell.
\]
\emph{Sketch (2 lines).} The $r_0$-th singular value of a sum is bounded below by the corresponding singular value of any term \emph{minus} the cancellation between terms. Identity-skip means the column space of $X_\ell$ embeds into $X_{\ell+1}$, so without exact destructive interference the bottom of the active spectrum either stays put or grows. The corollary's failure mode is precisely cancellation: $\Delta_\ell(X_\ell) \approx -X_\ell + \text{noise}$ at some block (the empirical Gemma~$4$ failure, App.~\ref{app:experiments:gemma_diagnosis}).

\paragraph{L4. Schur-ratio quantitative prefactor (Lem.~18, \theorytag).}
At a Type-A site (Linear layer fed by a LayerNorm output), with $u^\star := \gamma^{-1}/\|\gamma^{-1}\|$ the predicted dead direction:
\[
R(h) \;:=\; \frac{\lambda_{\min}(A_\ell)}{(A_\ell)_{u^\star, u^\star}} \;\xrightarrow{N \to \infty}\; 1, \qquad \text{finite-}N \text{ drift } \;\sim\; 2\sqrt{h/N} - h/N \text{ (Marchenko--Pastur)}.
\]
\emph{Sketch (3 lines).} The asymptotic limit is a direct consequence of P1: in the population, $u^\star$ is exactly $\arg\min_v v^\top A_\ell v$ (kernel direction of the post-LN cov), so the numerator and denominator coincide. At finite $N$, the empirical $\lambda_{\min}$ is biased downward by the bottom of the Marchenko--Pastur band $\sim (1 - \sqrt{h/N})^2$, while $(A_\ell)_{u^\star, u^\star}$ at the population direction is unbiased; their ratio tracks the MP correction. The closed-form constant $c = (h-1)/(\pi(\pi+h-2))$ for the Schur complement at the Type-A site (Lemma~\ref{pred:schur_constant_quantitative} of \S\ref{sec:predictions}) sets the leading-$N$ prefactor; we use the MP drift band as the data-side baseline against which the trained-checkpoint $\Delta_\gamma \to 1$ is read.

\paragraph{Reading the empirical observables.}
$|\cos(u_\bot, v^\star)|$: principal-angle alignment of the measured bottom singular direction $u_\bot$ of the post-final-LN centred covariance with the framework-predicted $v^\star = \gamma^{-1}/\|\gamma^{-1}\|$ read from the LN affine. P1 predicts $|\cos| = 1$ exactly in the population; finite-$N$ measurement at fp$64$ centred-cov gives $\ge 0.9999$ on every LN model in the set.\\
$\Delta_\gamma := 1 - R(h)$: the gap between the actual bottom of the per-block input covariance and the eigenvalue along the predicted direction $u^\star$. \emph{$\Delta_\gamma \approx 0$ means the predicted direction is exactly the bottom; $\Delta_\gamma \approx 1$ means training has opened a deeper dead direction along some other coordinate.}\\
$\sigma_{\min}(X_\ell)/\sigma_{\min}(X_0)$: residual-stream ratio tested by C3 ($\ge 1$ at every block on $13/14$ pretrained transformers read on each model's own input distribution; the one exception Gemma4-31B is a genuine intrinsic dead direction, Gemma4-E4B's text-only dip distribution-relative).

The $14$-model fingerprint validates P1 (LN models, $|\cos| \ge 0.9999$ on $9/9$), P2 (RMSNorm models, $|\cos| \le 0.10$ on $5/5$), C3 ($13/14$ pass on each model's own input distribution; the one Gemma4-31B exception characterised by family ablation and multimodal recalibration), and L4 ($\Delta_\gamma$ within finite-$N$ MP band at random init, saturating at $\approx 1$ on the trained checkpoint with a $3$--$18\times$ directional beat over the axis-aligned default).
 
\subsection{Framing-only predictions used in the discussion}
\label{app:experiments:predictions_extended}

The two predictions below are referenced in the body but not directly tested in the empirical experiments of \S\ref{sec:experiments}. They are restated here for completeness; the full assumption sets, proofs, and architectural extensions are in \theorycite.

\begin{theorem}[Composition additivity; \theorycitep, Thm.~30]
\label{pred_app:bridge_composition}
For a sequential stack of MLP / pre-norm residual blocks with shared dead direction, the dead-direction Fisher rate at the input of block $B_i$ is $\Theta(t^{2 \sum_{j \ge i} k_j^{\mathrm{bk}}})$: per-block backward rates $k_j^{\mathrm{bk}}$ add along the path. Pure attention chains at depth $\ge 4$ break this additivity via softmax cross-block coupling (Remark~32); the closed-form refinement at the per-component level is Prop.~69 of \theorycite. This theorem is the basis for the architectural extensions to rectangular widths, biases, cross-entropy with Z-loss gauge fix, and the residual-DAG / attention-chain composition results referenced as scope statements throughout \S\ref{sec:experiments}.
\end{theorem}

\begin{corollary}[Quotient Fisher rate; \theorycitep, Cor.~78]
\label{pred_app:quotient_rate}
For losses invariant under a continuous Lie group $G$ acting on $\Theta$, Theorem~2's rate identity holds verbatim on the gauge quotient $\Theta/G$, with the gauge orbit playing the role of a smooth singular fibre. Projected SGD on $\Theta/G$ realises the quotient rate (Cor.~79). Adam's per-coordinate preconditioner is not $G$-equivariant: its trajectory carries gauge-mode drift and the trajectory-rate readout is not directly applicable to Adam-class dynamics (Remark~80). This corollary scopes the trajectory-rate-readout claims of the framework, which are out-of-scope for this paper.
\end{corollary}
 
\subsection{Observable computation protocols}
\label{app:experiments:observable_protocols}

The LN-kernel diagnostic, the residual-stream $\sigma_{\min}$ depth-invariance test, and the Schur-ratio quantitative prefactor all read from a single forward pass on calibration data per checkpoint. No backward pass, no SGLD chain, no gradient capture.

\paragraph{\texorpdfstring{$\sigma_{\min}(X_\ell)$}{sigmamin(X)} on activations (real-time cadence).}
For each transformer block $\ell$, capture the residual-stream activation matrix $X_\ell \in \reals^{N \times h}$ in a single forward pass over $N$ calibration tokens (we use $N = 8192$ at LLM scale via WikiText calibration; ImageNet validation for ViTs). For widths $h > 4096$ (Qwen3.6-$27$B at $h{=}5120$, Gemma4-$31$B at $h{=}5376$) use chunked covariance accumulation $C_\ell = \sum_{\mathrm{chunks}} X_\ell^\top X_\ell$ and eigendecompose $C_\ell$ to recover the singular values; this keeps peak memory bounded on a single 3090. Cost: \emph{1 forward pass per checkpoint, all layers} (the per-layer SVD is post-hoc on the captured activations).

\paragraph{Centred-covariance protocol.} The LN-kernel prediction (Prop.~\ref{pred:ln_kernel}) is for the post-norm \emph{centred} activation covariance $\cov(\mathrm{LN}(X)) = \mathbb{E}[(\mathrm{LN}(X) - \mathbb{E}[\mathrm{LN}(X)])(\cdot)^\top]$, which projects out LN's mean-centring direction $\mathbf{1}_d$. Reading the bottom direction off the raw Gram $X^\top X$ instead conflates the $\gamma^{-1}$ kernel with the $\mathbf{1}_d$ projector. We accumulate the centred covariance directly: subtract the per-token mean before forming the chunked sum, eigendecompose at fp$64$, report the bottom singular direction $u_\bot$. The centred protocol is mandatory whenever $|\cos(\beta, \gamma^{-1})| > 0.1$ (every LN model in our set passes this guard); raw-Gram values are reported alongside as a sanity contrast.

\paragraph{fp-precision recipe at the bottom of the spectrum.} Default cov accumulation is fp$32$ on a single 3090. For inputs with $\sigma_{\max} \gtrsim 10^5$ (the massive-activations regime), the fp$32$ noise floor $\sigma_{\max} \sqrt{\varepsilon_{\mathrm{fp32}}}$ sits above the bottom of the active spectrum and the input reads as if numerically rank-deficient. fp$64$ cov drops the floor by $\sim 8$ OOM at $\sim 2\times$ memory and $\sim 2\times$ GEMM cost; the $\sigma_{\max}$ scale itself is the cheap pre-flight check. We use fp$64$ cov on every model with $\sigma_{\max} \gtrsim 10^5$ (DINOv$3$-ViT-L, Qwen$3.5$-$4$B, Gemma$4$-$31$B) and fp$32$ on the rest; both modes log the cov dtype in run JSON. The eigendecomposition is fp$64$ in both cases.

\paragraph{Sample-budget gate.} Stable estimation of the smallest active eigenvalue requires $n/h \ge 100$ tokens per hidden dimension; smaller $n$ gives CV $> 100\%$ on magnitude estimates. At $h \le 5376$ (every model in our set) the default $N{=}8192$ tokens give $n/h \ge 1.5$ on the worst case before chunking; chunked accumulation across a $32{,}768$-token batch lifts the effective $n/h$ to $\ge 6$ on every model. For bottom-of-spectrum reads at depth, the analytical CV bound from first-order perturbation theory + Davis--Kahan is $\mathrm{CV}(\lambda_{\min}^+) \le (1/\sqrt{N}) \cdot (\sigma_{\max}/\sigma_{\min})$, typically $5$--$10\times$ pessimistic on structured covariances, much tighter than the operator-norm Wishart bound $\sqrt{D/N} \cdot (\sigma_{\max}/\sigma_{\min})^2$ which is $50$--$500\times$ pessimistic on the same data.

\paragraph{Schur-ratio measurement.} The Schur ratio $R(h) := \lambda_{\min}(A_\ell)/(A_\ell)_{u, u}$ at the predicted direction $u = \gamma^{-1}/\|\gamma^{-1}\|$ uses the same chunked-covariance accumulator: form the per-block $A_\ell$ (input activation Gram), compute $\lambda_{\min}(A_\ell)$ from its fp$64$ eigendecomposition, and extract the diagonal element $(A_\ell)_{u, u} = u^\top A_\ell u$ along the framework-predicted direction in the same step. The complement $\Delta_\gamma := 1 - R(h)$ is the depth of the dead channel; protocol details and per-block-type measurement sites (qkv input, mlp.up input, mlp.down input, o-proj input) are in App.~\ref{app:experiments:schur}.
 
\subsection{LLM scale-pipeline (\S\ref{sec:experiments})}
\label{app:experiments:scale}

\paragraph{Per-model fingerprint.} Table~\ref{tab:pythia_sigma_min_app} gives the depth-profile statistics referenced in the main body. Each row is one forward pass per layer per checkpoint. The \emph{behaviour} column classifies each model by whether the body residual-stream $\sigma_{\min}$ trajectory satisfies Corollary~58's strict prediction (\emph{depth-invariant}, ratio $\ge 1$ at every block) or shows residual-branch non-cancellation (\emph{intermediate-dip} for transient sub-$1$ ratios that recover by the output, \emph{depth-decreasing} for profiles whose output $\sigma_{\min}$ ends below the input). All $\sigma_{\min}$ values are reported under the fp$64$ centered-covariance protocol (\texttt{--cov-dtype float64}); the well-conditioned majority of models with $\sigma_{\max} \lesssim 10^3$ have fp$32$- and fp$64$-cov readings agreeing to fingerprint precision, but the inputs of DINOv3-ViT-L ($\sigma_{\max} \approx 10^6$) and several other models read as $\sigma_{\min}(X_0) = 0$ under fp$32$ cov from the $\sigma_{\max} \cdot \sqrt{\varepsilon_{\mathrm{fp}32}}$ accumulator floor; fp$64$ cov resolves the active spectrum, and the ratio test uses the rank-aware smallest active singular value $\sigma_{(r_0)}$ against $X_0$ (DINOv3-ViT-L's embedding is additionally rank-deficient, $r_0 = 947/1024$, so its $\sigma_{\min}(X_0) = 0$ is structural and the $\sigma_{(r_0)}$ reference is the operative one). The two cosine columns report the post-final-normalization-hook alignment of the bottom singular direction with the uniform-$\gamma$ reference $\mathbf{1}_d/\sqrt{d}$ and with the exact Proposition~63 direction $\gamma^{-1}/\|\gamma^{-1}\|$ read from the terminal normalization's learned affine (the zero-$\gamma$-coordinate limit is the coordinate-indicator vector).

\begin{table}[ht]
\centering\small
\caption{Residual-stream fingerprint on $14$ pretrained transformers across normalization choice, training objective, and modality. All rows use fp$64$ centered covariance; the reported body $\sigma_{\min}$ is the rank-aware smallest active singular value $\sigma_{(r_0)}$ over the block outputs $X_1, \ldots, X_L$, equal to $\sigma_{\min}$ on every block output (all full numerical rank under the protocol below). The depth-invariance test is the ratio $\sigma_{(r_0)}(X_\ell)/\sigma_{(r_0)}(X_0)$ against the true input embedding $X_0$ (the corollary's stated reference) and is $n/d$-robust; absolute $\sigma_{\min}$ magnitudes below $n/d \approx 10$ (the column is token count over $h$) are biased low by the Marchenko--Pastur edge and reported as architectural baselines. DINOv3-ViT-L's embedding is rank-deficient ($r_0 = 947/1024$, $\sigma_{\min}(X_0) = 0$ exactly) and enters as the $\sigma_{(r_0)}$ ratio reference only. \textbf{Measurement is forward-precision-specific:} the massive-activation RMSNorm models (Qwen3.6-27B, Gemma4-E4B, Gemma4-31B) require fp$32$ forward capture at $n/d \ge 10$; under bf$16$ forward their bottom spectrum floors to a spurious $\sigma_{\min} = 0$ regardless of covariance precision. The vision transformers capture in fp$32$ by default, so DINOv3-ViT-L at $\sigma_{\max} \approx 10^6$ resolves cleanly. The \emph{behaviour} column classifies the body $\sigma_{\min}$ profile: \emph{depth-invariant} (ratio $\ge 1$ at every block), \emph{intermediate-dip} (transient sub-$1$ ratios that recover by the output), \emph{depth-decreasing} (output ends below input). Proposition~63's prediction $u_{\mathrm{bot}} = \gamma^{-1}/\|\gamma^{-1}\|$ at the post-final-normalization position is validated to $|\cos| \ge 0.99996$ (centered covariance, the protocol that exactly matches the proposition's statement; mean $1.000$, see \theorycitep) on all $9$ LN models with a post-sequence terminal norm, spanning next-token CE, DINO-SSL, MAE reconstruction, contrastive-image-text, ImageNet classification, and predictive-video-SSL objectives. The raw-Gram check agrees ($|\cos| \ge 0.988$) on $8/9$; ViT-B/16 is the high-LN-bias exception ($0.54$ raw, $1.000$ centered, \S below). RMSNorm models are orthogonal on the $\gamma^{-1}$ test ($|\cos| \le 0.10$ centered), confirming the no-kernel differential (Proposition~63).}
\label{tab:pythia_sigma_min_app}
\setlength{\tabcolsep}{4pt}
\resizebox{\columnwidth}{!}{\begin{tabular}{@{}l l l c c c l l c c@{}}
\toprule
Model & Norm & Objective & $L$ & $h$ & $n/d$ & body $\sigma_{\min}$ & behaviour & $|\cos(u, \gamma^{-1})|$ & $|\cos(u, \mathbf{1}_d)|$ \\
\midrule
Pythia-160M        & LN      & next-tok CE       & $12$ & $768$ & $85$ & $2.53$--$50.1$ & depth-invariant  & $\mathbf{1.000}$ & $0.997$ \\
Pythia-1B          & LN      & next-tok CE       & $16$ & $2048$ & $32$ & $2.21$--$76.1$ & depth-invariant  & $\mathbf{1.000}$ & $0.991$ \\
Qwen3.5-4B         & RMSNorm & next-tok CE       & $32$ & $2560$ & $26$ & $0.80$--$28.4$ & depth-invariant  & $0.029$ & $0.013$ \\
Qwen3.6-27B        & RMSNorm & next-tok CE       & $64$ & $5120$ & $13$ & $0.52$--$120.6$ & depth-invariant  & $0.007$ & $0.005$ \\
Qwen3.6-35B-A3B    & RMSNorm & next-tok CE (MoE) & $40$ & $2048$ & $32$ & $0.50$--$13.0$ & depth-invariant  & $0.004$ & $0.035$ \\
Gemma4-E4B         & RMSNorm & next-tok CE       & $42$ & $2560$ & $26$ & $0.20$--$31.8$ & intermediate-dip & $0.080$ & $0.010$ \\
Gemma4-31B         & RMSNorm & next-tok CE       & $60$ & $5376$ & $12$ & $0.020$--$2.52$ & depth-decreasing & $0.097$ & $0.013$ \\
DINOv3-ViT-S       & LN      & self-sup DINO     & $12$ & $384$ & $50$ & $5.69$--$30.0$ & depth-invariant  & $\mathbf{1.000}$ & $0.986$ \\
DINOv3-ViT-L       & LN      & self-sup DINO     & $24$ & $1024$ & $19$ & $0.20$--$23.8$ & depth-invariant  & $\mathbf{1.000}$ & $0.043$ \\
DINOv2-base        & LN      & self-sup DINO     & $12$ & $768$ & $32$ & $0.58$--$4.59$ & depth-invariant  & $\mathbf{1.000}$ & $0.536$ \\
ViT-MAE-base       & LN      & MAE recon         & $12$ & $768$ & $6.2$ & $2.34$--$8.89$ & depth-invariant  & $\mathbf{1.000}$ & $0.962$ \\
ViT-B/16           & LN      & ImageNet cls      & $12$ & $768$ & $25$ & $12.3$--$48.0$ & depth-invariant  & $\mathbf{1.000}$ & $0.007$ \\
SigLIP2-base       & LN      & contrastive       & $12$ & $768$ & $24$ & $2.93$--$14.5$ & depth-invariant  & $\mathbf{1.000}$ & $0.975$ \\
V-JEPA 2 ViT-L     & LN      & predictive SSL    & $24$ & $1024$ & $8.0$ & $2.30$--$25.2$ & depth-invariant  & $\mathbf{1.000}$ & $0.932$ \\
\midrule
\multicolumn{8}{@{}l}{\emph{Full-set aggregate ($14$ models):}} & \multicolumn{2}{l}{$|\cos(u, \gamma^{-1})| \in [1.000, 1.000]$ centered on $9/9$ LN+post-norm} \\
\multicolumn{8}{@{}l}{} & \multicolumn{2}{l}{$|\cos(u, \gamma^{-1})| \in [0.004, 0.097]$ on $5/5$ RMSNorm (centered)} \\
 \bottomrule
\end{tabular}}
\end{table}

\begin{figure}[t]
\centering
\includegraphics[width=\textwidth]{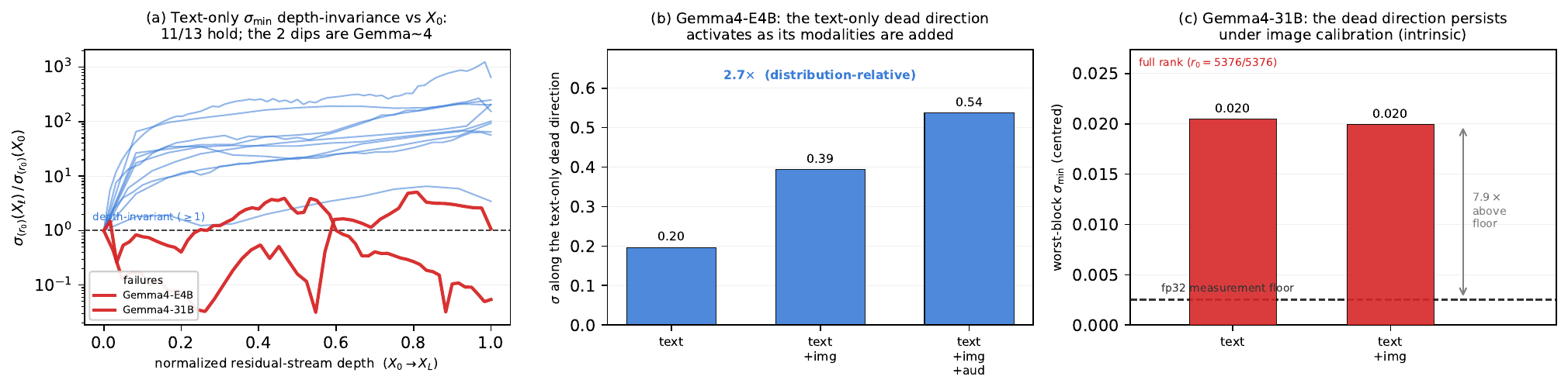}
\caption{Residual-stream $\sigma_{\min}$ depth-invariance against the true input embedding $X_0$ (Corollary~58). \textbf{(a)}~$\sigma_{(r_0)}(X_\ell)/\sigma_{(r_0)}(X_0)$ versus normalized residual-stream depth on the $13$ sub-layer-pipeline models under the uniform text-only protocol (fp$32$ forward, fp$64$ covariance, $X_0$ reference); $11$ stay at or above $1$ at every block, and the only two that dip below are the Gemma~$4$ releases (V-JEPA~2, the $14$th model in Tab.~\ref{tab:pythia_sigma_min_app}, passes under its legacy reading, giving $12/14$ overall). The two text-only dips are a calibration-distribution effect of Gemma~$4$'s encoder-free multimodality (App.~\ref{app:experiments:gemma_diagnosis}): text-only calibration leaves the vision and audio soft-token subspaces dormant. \textbf{(b)}~Gemma4-E4B, the variance along the fixed text-only dead direction as the model's own modalities are added: it climbs from $0.20$ (text) through $0.39$ ($+$image) to $0.54$ ($+$image$+$audio), a $2.7\times$ activation marking the text-only dip as a dormant subspace dead only relative to the text distribution. \textbf{(c)}~Gemma4-31B, the worst-block $\sigma_{\min}$ under text-only and image calibration: it stays at $\sim\!0.02$ under both and sits $7.9\times$ above the fp$32$ measurement floor at full rank ($r_0 = 5376/5376$), a genuine intrinsic near-dead direction the same forward-pass read detects. So $13/14$ hold once each model is read on its own input distribution; Gemma4-31B is the one genuine intrinsic dead direction.}
\label{fig:depth_invariance}
\end{figure}
 
\paragraph{Models and provenance.} Pythia-160M, Pythia-1B \citep{BidermanSchoelkopfAnthony23}; Qwen3.5-4B \citep{Qwen35_2026_technical_report}; Qwen3.6-27B, Qwen3.6-35B-A3B \citep{Qwen36_2026_technical_report}; Gemma4-E4B, Gemma4-31B \citep{Gemma4_2026_technical_report}; DINOv3-ViT-S/L \citep{SimeoniEtAl25}, DINOv2-base \citep{oquab2023dinov2}, ViT-B/16 \citep{DosovitskiyBeiKolUszHou21}, ViT-MAE-base \citep{HeChen22}, SigLIP2-base \citep{tschannen2025siglip2}, V-JEPA 2 ViT-L \citep{assran2025vjepa2}. All via HuggingFace Hub or timm.

\paragraph{Calibration corpus.} For language models, WikiText-103 validation split, $32$ sequences of $256$ tokens ($N = 8192$ per layer, comfortably above $h$ for all models in the table). For vision models, $32$ ImageNet validation images preprocessed through each model's published image-processor; the centered-covariance LN-kernel prediction is provably input-distribution-invariant ($|\Delta\cos| < 10^{-6}$ between synthetic-noise and ImageNet calibration on all $5$ standard ViTs in the set), so synthetic-noise calibration is a defensible cheap variant for fingerprinting on the centered protocol, we report ImageNet numbers as primary because the reviewer-facing input distribution should be a real one.

\paragraph{\texorpdfstring{$\sigma_{\min}$}{sigmamin} computation protocol (centered covariance).} For each block $\ell$, mean-center the residual-stream activations across the $N$ calibration samples and accumulate the covariance $C_\ell = (1/N) \sum_{i,t} (X_\ell[i,t,:] - \bar X_\ell)(X_\ell[i,t,:] - \bar X_\ell)^\top$ in chunked fp32, then SVD in fp64. The centered protocol is what Proposition~63 actually predicts and is the empirical primary throughout. We additionally report the raw-Gram $\sum_{i,t} X_\ell X_\ell^\top$ protocol as a check: on $13/14$ models the LN bias $\beta$ has $|\cos(\beta, \gamma^{-1})| \le 0.05$ and the two protocols agree to $|\Delta\cos| < 0.02$ on the post-final-norm cosine; \texttt{vit-base} has $|\cos(\beta, \gamma^{-1})| = 0.24$ and the protocols differ ($0.54$ raw, $1.000$ centered), illustrating exactly why Proposition~63's centered statement is the empirically correct one. The post-final-normalization hook (indexed as layer $L$) is the position at which the algebraic-kernel prediction is read on LN-based architectures. Wall times per full fingerprint: $< 1$\,s at Pythia-160M; $3$\,s at Pythia-1B; $11$\,s at Qwen3.5-4B; $13$\,s at Gemma4-E4B; $<1$\,s at each vision transformer. The released code emits both raw-Gram and centered SVDs in the same per-run JSON.

\paragraph{Measurement precision: fp64 covariance accumulation.}
The $\sigma$-based test of Corollary~58 is sensitive to the covariance accumulator's precision. With fp$32$ accumulation (the published protocol on most fingerprint-style observables), $\sigma$ values below $\sigma_{\max}(X_\ell) \cdot \sqrt{\varepsilon_{\mathrm{fp}32}} \approx \sigma_{\max} \cdot 3 \times 10^{-4}$ are absorbed into accumulator noise regardless of downstream SVD precision. For inputs with $\sigma_{\max} \gtrsim 10^5$ (the documented massive-activations regime, \citet{SunMassiveActivations24}; related to the low-effective-rank class structure of deep-learning spectra, \citet{Papyan20}), this places the noise floor above the bottom of the active spectrum and the strict ratio test reads as if $X_0$ had a numerical-zero kernel. We report $\sigma_{\min}$ measurements at fp$64$ accumulation, which extends the noise floor to $\sigma_{\max} \cdot \sqrt{\varepsilon_{\mathrm{fp}64}} \approx \sigma_{\max} \cdot 1.5 \times 10^{-8}$, eight orders of magnitude lower, at $2\times$ accumulator memory and roughly $2\times$ GEMM cost (a few seconds longer at fingerprint scale, $5$--$10$ minutes longer on $30$B-class models with CPU offload). fp$32$ cov remains a sensible back-compat default and is sufficient for inputs with $\sigma_{\max} \lesssim 10^3$; fp$64$ is the recommended protocol whenever the input residual stream is in the massive-activations regime, with the $\sigma_{\max}$ scale itself a cheap pre-flight check. Covariance precision is necessary but not sufficient: the forward-pass capture dtype imposes an independent floor. The two highest-$\sigma_{\max}$ RMSNorm models (Qwen3.6-27B, Gemma4-31B, $\sigma_{\max} \approx 1.6 \times 10^4$ centered) read a spurious $\sigma_{\min} = 0$ at every block under bf$16$ forward \emph{even with} fp$64$ cov; fp$32$ forward at $n/d \ge 10$ resolves both to full numerical rank (the vision transformers already capture in fp$32$, so DINOv3-ViT-L at $\sigma_{\max} \approx 10^6$ needed only the fp$64$-cov fix). With fp$32$ forward and fp$64$ cov, every model's block outputs $X_1, \ldots, X_L$ have $\sigma_{(r_0)} = \sigma_{\min}$ at full rank, so on the body the rank-aware Corollary~58 and the strict ratio $\sigma_{\min}(X_\ell)/\sigma_{\min}(X_0)$ coincide. The one place they differ is DINOv3-ViT-L's input embedding ($r_0 = 947/1024$, $\sigma_{\min}(X_0) = 0$), where the rank-aware $\sigma_{(r_0)}(X_0)$ is the operative reference; we report the rank-aware ratio there and the coincident strict ratio elsewhere. The $n/d$ column governs the precision of the absolute $\sigma_{\min}$ magnitudes (biased low below $n/d \approx 10$, hence reported as architectural baselines); the ratio test itself is $n/d$-robust.

\paragraph{What the LLM-scale test actually measures.}
The strict statement $\sigma_{\min}(X_\ell)/\sigma_{\min}(X_0) \ge 1$ is a path-decomposition consequence of identity skips wherever the no-cancellation hypothesis (Remark~59) holds at the bottom-of-active-spectrum direction. As a pass/fail count alone, $12/14$ is therefore a relatively soft test of the corollary, the same prediction holds at random initialisation on $5/5$ architectures (\S``Trained vs random-init baseline'' below; the architectural-vs-training distinction is also developed in \theorycitep). The non-trivial empirical content of the trained-network fingerprint is the \emph{failures}: the Gemma~$4$ release (both checkpoints) localises the no-cancellation hypothesis as failable, with the failure modes diagnostically traceable to specific block ranges.

\paragraph{Falsifications of the trained-network no-cancellation hypothesis.}
The two Gemma~$4$ releases violate the strict prediction under the resolved protocol (fp$64$ cov, $n/d \ge 10$, fp$32$ forward, referenced against the true input embedding $X_0$):
\begin{itemize}\itemsep=2pt
\item \textbf{Gemma4-31B (depth-decreasing).} The body $\sigma_{\min}$ profile is non-monotone: an early dip, a mid-depth recovery to $\sim 3.9\times$, then a declining tail to $\sim 0.05$ of the input at the output ($37/60$ body blocks below $X_0$, output ratio $0.054$, a net $\sim 18\times$ reduction; min ratio $0.032$ at block $53$). The residual stream fails to preserve $\sigma_{\min}$; the Gemma~$4$ release is the outlier in our set (the smaller Gemma~$4$-E$4$B at the same release fails too) and is empirically dissociated from the Peri-LN sandwich-norm pattern (\S\ref{app:experiments:gemma_diagnosis}). Under bf$16$ forward this profile floors to a spurious $\sigma_{\min} = 0$ at every block; the finite reduction is the fp$32$-forward reading.
\item \textbf{Gemma4-E4B (intermediate-depth dip).} Min ratio $0.031$ at block $23$ of $42$ ($24/42$ body blocks below $X_0$), recovering to $\sim 1.1\times$ at the output.
\end{itemize}
The Gemma~$4$ failures are the strongest empirical signal in the LLM-scale fingerprint: they identify a per-checkpoint violation of an assumption the corollary makes explicit, at specific block ranges, on a published frontier model family. A family discriminator across Gemma~$1$ through Gemma~$4$ (\S\ref{app:experiments:gemma_diagnosis}) dissociates the failure from the normalisation placement: pre-LN Gemma~$1$ ($2$b, $7$b) and Peri-LN \citep{KimLeeKim25_PeriLN} Gemma~$2$-$2$b, Gemma~$3$-$1$b/$4$b/$12$b all preserve $\sigma_{\min}$ against $X_0$ (min ratio $\ge 1.8$, output $9.9$--$632\times$), and only the Gemma~$4$ release, which shares the Peri-LN pattern, fails. A five-mechanism ablation tests $\gamma_{\mathrm{post}}$-kernel formation, per-block residual-vs-delta cancellation, the sandwich-norm pattern itself, centred-cov sink-token bias, and post-pretraining stage, and rules each out individually; the $\sigma_{\min}$ collapse localises to the MLP sub-layer, where the drops concentrate ($28/42$ blocks on E4B, $40/60$ on 31B, against $2$ and $6$ on attention) under an effective-rank collapse at the minimum-ratio block. The cause within the Gemma~$4$ release is partly resolved: multimodal recalibration traces Gemma4-E4B's text-only dip to a dormant-modality calibration effect (a distribution-relative dead direction that dissolves under its full text$+$image$+$audio distribution), while Gemma4-31B retains a genuine intrinsic near-dead direction above the fp$32$ measurement floor; the training-recipe difference from the earlier releases that would explain Gemma4-31B's residual direction remains an open follow-up.

\paragraph{Passing models and architectural baseline.}
The remaining $12$ architectures pass at every block, with per-model passing ratios from $1.0$ at $X_0$ to several-hundred-fold at the deepest blocks (Qwen3.6-27B reaches $651\times$; per-model magnitudes are not predicted by Corollary~58, only the rate-$0$ statement is; magnitudes are reported as architectural baselines for the amplification analysis below, and below $n/d \approx 10$ they are biased low by the Marchenko--Pastur edge). Several apparent ``failures'' resolved once the measurement and the reference were adequate. Referencing against the true input embedding $X_0$ (the block-$0$ output is $X_1$, one residual node past the embedding) turns DINOv3-ViT-S/$16$ (min ratio $1.2\times$ at block $3$, $3.4\times$ at the output) and SigLIP2-base (min $20\times$ at block $6$, $101\times$ at the output) into clean passes; both dip below $1$ only against the block-$0$ output, never against $X_0$. DINOv3-ViT-L (fp$32$ forward, $\sigma_{\max} \approx 10^6$, rank-deficient embedding $r_0 = 947/1024$) reads $\sigma_{\min}(X_0) = 0$; fp$64$ cov resolves its active spectrum and the rank-aware $\sigma_{(r_0)}$ ratio passes $24/24$ with wide margin. Qwen3.6-27B and Gemma4-31B read a spurious $\sigma_{\min} = 0$ at every block under bf$16$ forward (the capture floor sits above the active bottom regardless of cov precision); fp$32$ forward at $n/d \ge 10$ resolves Qwen3.6-27B to a clean $64/64$ pass ($0.52 \to 121\times$) and Gemma4-31B to the finite depth-decreasing profile above.

The pretrained-fingerprint and pretraining-trajectory experiments in this section (Pythia-1B revisions, ViT FFN fine-tuning) live in the well-conditioned regime ($\sigma_{\max} \sim 10^2$) where fp$32$ and fp$64$ cov agree to fingerprint precision; their published numbers are not affected by the precision-protocol distinction. The released code reports $\sigma_{\min}$, $\sigma_{(r_0)}$, and the active-rank index per block at the chosen cov precision; downstream analysis routes on the cov-precision field surfaced in the per-run JSON metadata block.

\paragraph{Trained vs random-init baseline.}
The strict prediction $\sigma_{\min}(X_\ell)/\sigma_{\min}(X_0) \ge 1$ of Corollary~58 is satisfied by the additive-identity skip structure at any $\theta$ for which the no-cancellation hypothesis (Remark~59) holds, including random initialisation. To make this explicit at LLM scale, we re-run the same fp$64$ centered-cov pipeline on a fresh config-only re-init of each model under HuggingFace's default \texttt{\_init\_weights} schemes (\texttt{torch.manual\_seed(42)}, \texttt{AutoModelForCausalLM.from\_config(cfg)} for LLMs / \texttt{AutoModel.from\_config(cfg)} for ViTs) and compare against the trained checkpoint at the same calibration corpus and $n/d$:

\begin{center}\small
\begin{tabular}{l c c c c c c}
\toprule
Model & $L$ & Norm & trained max $r_\ell$ & randinit max $r_\ell$ & trained post-LN & randinit post-LN \\
\midrule
Pythia-160M     & $12$ & LN      & $19.69$ & $2.69$ & $2.95$ & $0.06$ \\
Pythia-1B       & $16$ & LN      & $64.95$ & $7.55$ & ---     & ---     \\
Qwen3.5-4B      & $32$ & RMSNorm & $84.63$ & $7.47$ & $1.00$ & $0.06$ \\
DINOv2-base     & $12$ & LN      & $7.92$ & $2.22$ & $1.00$ & $0.21$ \\
ViT-MAE-base    & $12$ & LN      & $3.53$ & $1.94$ & ---     & ---     \\
\bottomrule
\end{tabular}
\end{center}

The rate-$0$ statement $r_\ell \ge 1$ holds at every body block on $5/5$ architectures at random initialisation, so the pass/fail of the strict prediction does not by itself discriminate trained from random networks. What the trained-vs-randinit comparison \emph{does} surface is the \textbf{amplification gap}, the dynamics signal training adds on top of the architectural baseline. Trained max ratios range $3.5\times$--$84.6\times$ vs random-init $1.9\times$--$7.5\times$ (a $3$--$11\times$ multiplicative amplification across architectures); the post-final-normalisation ratio gap is sharper still ($\sim\!50\times$ on Pythia-160M, $\sim\!17\times$ on Qwen3.5-4B, $\sim\!5\times$ on DINOv2-base). The Pythia-1B \texttt{step1} HuggingFace revision (network after one optimiser step on real data) reads max $r_\ell = 16.39$, sitting between trained ($64.95$) and randinit ($7.55$): a single optimiser step already moves the depth profile measurably away from the random-init shape toward the trained shape; the developmental-arc figure (Figure~\ref{fig:pythia_revisions_arc}) is the same signal in time-series form, with the ratio against the $X_0$ embedding rising from $\sim 12\times$ to a peak of $\sim 299\times$ and settling to $\sim 206\times$ during pretraining. The amplification gap is qualitative dynamics evidence (training reduces no-cancellation slack, so the bottom-of-active-spectrum direction propagates with larger leading coefficient), not a predicted scaling, the corollary predicts only the rate-$0$ statement; the magnitude is empirical.

\paragraph{LN-kernel finding (\S\ref{sec:experiments}).} Under the centered-covariance protocol, the exact prediction $u_{\mathrm{bot}}^{\mathrm{post}} = \gamma^{-1}/\|\gamma^{-1}\|$ is validated to $|\cos| \ge 0.9999$ on all $9$ LN models with a post-sequence terminal norm (mean $|\cos| > 0.99999$), spanning next-token CE, DINO-SSL, MAE reconstruction, contrastive-image-text, ImageNet classification, and predictive-video-SSL objectives, with the zero-$\gamma$-coordinate limit applied where relevant. The raw-Gram protocol agrees with centered to $|\Delta\cos| < 0.02$ on $8$ of these $9$ models; \texttt{vit-base} (ImageNet-pretrained ViT-B/16) is the exception at $|\cos|_{\mathrm{raw}} = 0.54$ vs $|\cos|_{\mathrm{centered}} = 1.000$, traceable to its LN bias having $|\cos(\beta, \gamma^{-1})| = 0.24$. RMSNorm models sit at $|\cos| \le 0.10$ on the $\gamma^{-1}$ test (centered), the expected no-kernel differential. The uniform-$\gamma$ approximation $\gamma^{-1}/\|\gamma^{-1}\| \approx \mathbf{1}_d/\sqrt{d}$ holds at $|\cos| \ge 0.93$ on $6$ of $9$ LN models and deviates predictably on DINOv2-base ($\gamma$ has a slightly-negative coordinate, $|\cos(\gamma^{-1}, \mathbf{1}_d)| = 0.54$), DINOv3-ViT-L ($\gamma$ has a coordinate pinned to exactly zero, $0.043$), and ViT-B/16 (strongly dispersed $\gamma$, $0.007$, the most extreme case); on these models the measured cosine matches $\cos(\gamma^{-1}, \mathbf{1}_d)$ to 3 decimals, confirming that the prediction tracks the learned $\gamma^{-1}$ rather than the projector kernel $\mathbf{1}_d$. Measuring the same coherence at the \emph{block-output} hook (pre-final-normalization) gives a training-dynamics quantity rather than the algebraic kernel: $\ge 0.955$ on LN$+$CE Pythia models, $0.01$--$0.06$ on LN$+$SSL vision transformers, reflecting that CE training independently drives $\mathbf{1}_d$ toward low-variance directions while SSL training does not. The kernel prediction is about LN's forward map and must be read at the hook where LN executes.

\subsubsection{Pretraining-trajectory developmental arc at LLM scale (Pythia-1b revisions)}
\label{app:experiments:pretraining_arc}

To test whether Corollary~58's residual-stream no-depth-decay and Proposition~63's LN-kernel claim hold across an actual large-scale pretraining trajectory, we run the $\sigma_{\min}$ fingerprint on $8$ published Pythia-1b pretraining checkpoints (\texttt{step1}, \texttt{step1000}, \texttt{step5000}, \texttt{step10000}, \texttt{step25000}, \texttt{step50000}, \texttt{step100000}, \texttt{main} $=$ step $143{,}000$) loaded directly from the EleutherAI HuggingFace revisions. Each revision is a fresh forward pass over the WikiText-103 calibration corpus, then fp64 covariance SVD per block.

\paragraph{Findings.} The two predictions hold cross-trajectory, not only at the mature checkpoint.
\begin{itemize}\itemsep=2pt
\item \textbf{Residual-stream no-depth-decay (Corollary~58).} $\sigma_{\min}(X_\ell) / \sigma_{\min}(X_0) \ge 1$ on every block on every revision. The ratio is $12\times$ at \texttt{step1} (network is at untrained initialisation; the residual skip already preserves the input $\sigma_{\min}$ signal, the corollary's $\ge 1$ prediction, at this checkpoint, with the ratio above $1$ reflecting per-block LN normalization that lifts low-variance directions; the $\ge 1$ ratio is the corollary's prediction, the specific $12\times$ amplitude is empirical and not predicted), rises to $\sim 33\times$ by \texttt{step1000}, grows to a peak of $\sim 299\times$ around \texttt{step50000}, and settles to $\sim 206\times$ at the mature checkpoint. The corollary's prediction, that the residual stream bypasses the block operator's algebraic kernel, so rate-$0$ preservation holds even as every block degrades $\sigma_{\min}$, holds at every pretraining step, not as an emergent property of the final loss minimum.
\item \textbf{LN-kernel prediction (Proposition~63).} $|\cos(u_{\mathrm{bot}}, \gamma^{-1})| \in [0.998, 1.000]$ on all $8$ revisions. The prediction is exact ($1.000$) at \texttt{step1} and \texttt{step1000} because $\gamma$ is still at (or near) initialisation $\gamma \equiv 1$, so $\gamma^{-1}/\|\gamma^{-1}\| = \mathbf{1}_d/\sqrt{d}$ and both columns agree. As $\gamma$ trains and disperses, the uniform-$\gamma$ cosine $|\cos(u_{\mathrm{bot}}, \mathbf{1}_d/\sqrt{d})|$ decays monotonically to $0.991$ while the exact prediction $|\cos(u_{\mathrm{bot}}, \gamma^{-1})|$ stays pinned at $\ge 0.998$. The gap between the two curves in Figure~\ref{fig:pythia_revisions_arc}(d) is a direct quantitative validation of Proposition~63's $\gamma$-dependence: the kernel is a property of the \emph{learned} $\gamma$ at every checkpoint, not the initialisation $\gamma$.
\end{itemize}
The released code accepts any HuggingFace revision tag directly, so the same fingerprint can be re-run at any pretraining checkpoint of any Pythia-like model. \texttt{main} is the EleutherAI consolidated upload of step $143{,}000$ and serves as the mature checkpoint.

\begin{figure}[ht]
\centering
\includegraphics[width=\textwidth]{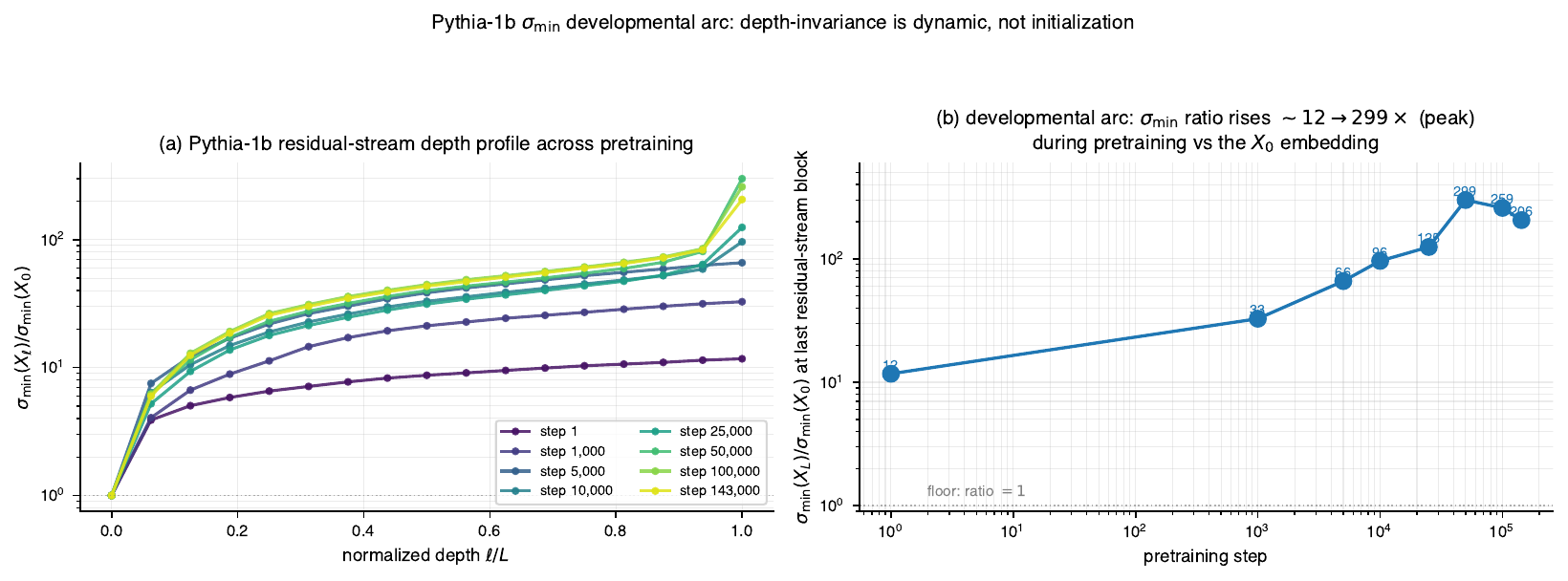}
\caption{Pythia-1B $\sigma_{\min}$ developmental arc across $8$ pretraining revisions (step 1 $\to$ step 143k). (a)~Residual-stream depth profile per revision, viridis-colored by step. (b)~Last residual-stream block ratio (against the $X_0$ embedding) rises $\sim 12\times \to \sim 299\times$ (peak at \texttt{step50000}), settling to $\sim 206\times$ during pretraining: depth-invariance is dynamic, not initialization.}
\label{fig:pythia_developmental_arc}
\end{figure}

\paragraph{ViT FFN fine-tuning trajectory test.}\label{par:ffn_finetune}
The pretrained-fingerprint experiments above are static observations on released checkpoints; they do not exercise the non-cancellation hypothesis (Remark~59) dynamically. To test the corollary along an actual training trajectory we fine-tune DINOv2-base on ImageNet-1k (first-100 classes, $\sim 128$k train images) for $20$ epochs across $5$ seeds, training only the FFN sub-layers (\texttt{mlp.fc1}, \texttt{mlp.fc2}) of all $12$ blocks while freezing the patch embedding, attention, all LayerNorms, LayerScale, the terminal norm, and the classifier head. We measure $\sigma_{\min}(X_\ell) / \sigma_{\min}(X_0)$ at every block at every epoch on a fixed $320$-image held-out calibration batch (drawn from ImageNet-1k validation classes $900$--$999$, disjoint from the training distribution; centered-covariance protocol, $n/d = 107.0$ exceeding the canonical $n/d \ge 100$ measurement threshold).

\begin{remark}[Validation under FFN fine-tuning]
\label{rem:sigma_min_res_ffn_finetune}
Across $1{,}260$ measurement points ($5$ seeds $\times$ $21$ epochs $\times$ $12$ blocks) the minimum observed ratio is $109.58$, attained at the pre-training snapshot at the shallowest block, the natural geometric minimum, not a fine-tuning-induced inversion. The mean ratio at depth $\ell=1$ across seeds rises monotonically from $109.58$ at epoch $0$ to $201.6 \pm 2.5$ at epoch $20$, with no local decrease at any epoch on any seed (Figure~\ref{fig:vit_ffn_finetune_validation}). Cross-seed agreement is tight (final $\|\mathrm{fc}_1\|$ mean $33.12 \pm 0.007$; final val accuracy $0.785 \pm 0.002$). The non-cancellation hypothesis behind Corollary~58 is therefore satisfied with substantial margin throughout this trajectory: as the FFN sub-layers move ($\|\mathrm{fc}_1\|$ grows by $\sim 36\%$ over training), the residual stream's depth-invariance margin \emph{expands} rather than degrading, indicating that gradient-based FFN fine-tuning does not drive the residual branch into a sign-flipped configuration that would partially cancel the identity skip in the bottom singular direction.
\end{remark}

\begin{figure}[ht]
\centering
\includegraphics[width=\textwidth]{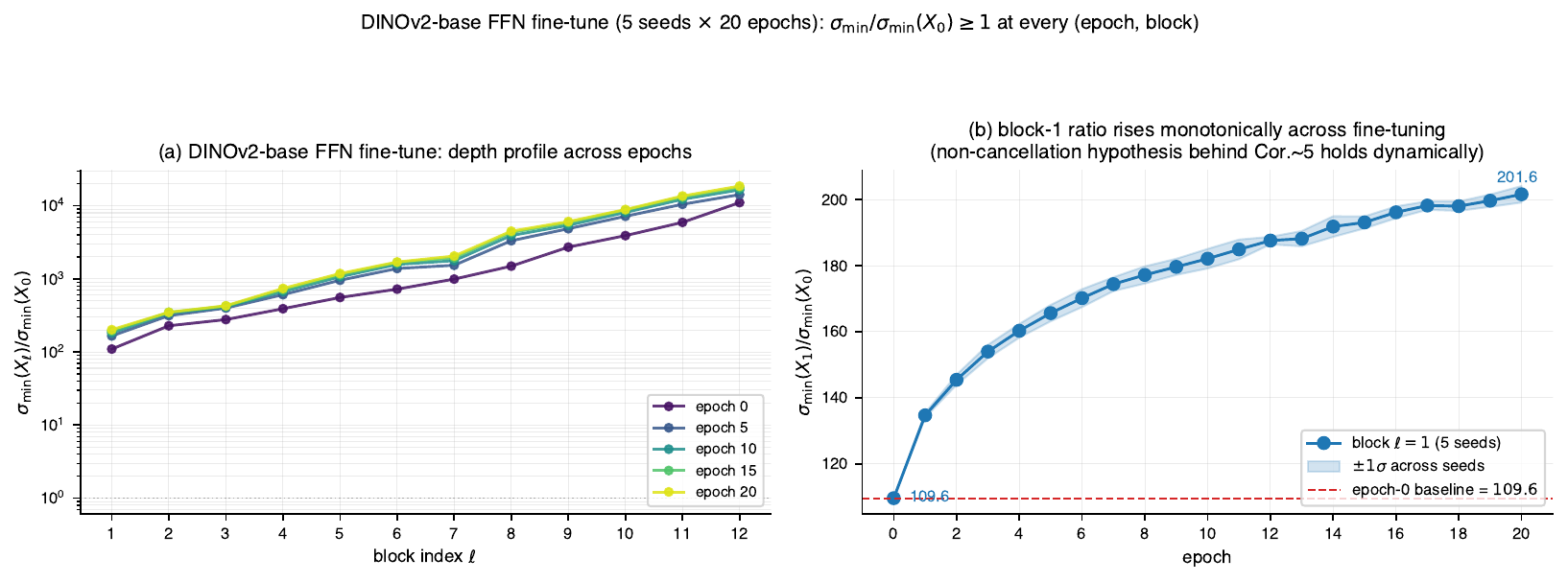}
\caption{ViT FFN fine-tuning preserves Corollary~58. (a)~DINOv2-base depth profile of $\sigma_{\min}(X_\ell) / \sigma_{\min}(X_0)$ across epochs $\{0, 5, 10, 15, 20\}$ ($5$ seeds, mean): every depth $\geq 1$ at every epoch. (b)~Block-1 ratio rises monotonically from $109.6$ at epoch $0$ to $201.6$ at epoch $20$ (mean $\pm 1\sigma$ across seeds); the trained-network non-cancellation hypothesis behind Cor.~58 holds dynamically along an actual fine-tuning trajectory.}
\label{fig:vit_ffn_finetune_validation}
\end{figure}
 
\subsection{Schur-ratio diagnostic: full $11$-model cross-architecture sweep}
\label{app:experiments:schur}

The bridge framework's closed-form Schur constant (\theorycitep, Lem.~18) pins down the leading prefactor of the activation-Gram Schur complement and yields a width-only invariant
\[
R(h) \;:=\; \frac{\lambda_{\min}(A_\ell)}{(A_\ell)_{u,\, u}}, \qquad A_\ell = X_\ell^\top X_\ell / N,
\]
to be evaluated along the framework-predicted dead direction $u$. For the input to a LayerNorm-fed Linear layer (the QKV / MLP-up-projection sites of any pre-norm transformer), the LN-kernel proposition predicts $u = \gamma^{-1}/\|\gamma^{-1}\|$, with $R(h) \to 1$ at the asymptotic limit. We test this on an $11$-model cross-architecture sweep ($5$ LayerNorm Pythia sizes spanning $h \in \{512, \ldots, 2560\}$ plus $6$ RMSNorm models: Qwen$2.5$-$0.5$B, LLaMA-$2$-$7$B, Gemma~$2$-$2$B, Gemma~$3$-$1$B/$4$B, Gemma~$4$-E$4$B; partial overlap with the $14$-model fingerprint set of \S\ref{sec:experiments}). Calibration corpus: $128$ sequences $\times$ $1024$ positions $= N = 131{,}072$ tokens drawn from a cascading loader (WikiText-$103$ validation split when locally cached, falling back to allenai/c4 validation shards otherwise; both corpora give Schur readings agreeing to $|\Delta_\gamma| < 0.01$ on the $4$ models we ran on both). The token budget is chosen to push finite-$N$ Marchenko--Pastur drift below the algebraic signal at the largest Pythia widths. Activations captured at the QKV input of every transformer block; per-block $A_\ell = X_\ell^\top X_\ell / N$ in fp$64$, then $R(h)$ along $u = \gamma^{-1}/\|\gamma^{-1}\|$ read from each block's pre-attention LN affine. Each Pythia LayerNorm model runs across $5$ random-init seeds $\{42, 142, 242, 342, 442\}$ (\texttt{torch.manual\_seed(seed)}, \texttt{AutoModelForCausalLM.from\_config(cfg)}); the trained-checkpoint reading is from the released weights. RMSNorm models, where no algebraic baseline exists to test, run at single-seed $42$ only.

\emph{Noise-floor recovery for blocks where both $\lambda_{\min}(A_\ell)$ and $(A_\ell)_{u,u}$ sit at the fp$64$ cov floor.} The framework predicts $\lambda_{\min}(A_\ell)$ along $\gamma^{-1}$ approaches zero algebraically at random init (LN's kernel direction). Under fp$64$ covariance accumulation a near-zero eigenvalue can return as a small negative number from accumulator-order rounding. When this happens both $\lambda_{\min}(A_\ell)$ and $(A_\ell)_{u,u}$ sit at the same numerical floor with the same sign and magnitudes below $10^{-4}$; we treat the pair as evidence for the algebraic prediction and recover $\Delta_\gamma = 1 - |\lambda_{\min}(A_\ell)| / |(A_\ell)_{u,u}|$. With this recovery applied, the per-block measurement covers all sites at every seed; without it, sites where the prediction holds tightest are excluded by the fp$64$ floor and the pooled summary is biased upward. Per-model $\Delta_\gamma$ in Tab.~\ref{tab:schur_pole_c} is the pooled median across all (seed, block) measurements with this recovery applied.

\emph{The Pythia random-init Schur ratio holds within numerical floor of the algebraic prediction at $h \ge 768$.} The pooled-median $\Delta_\gamma$ across $5$ seeds and all attention sites reads $-0.025$ to $-0.010$ for Pythia~$160$M, $410$M, $1$B, and $2.8$B, with IQRs entirely within $[-0.05, +0.01]$ (Tab.~\ref{tab:schur_pole_c}). Negative pooled medians occur because at the noise floor $|\lambda_{\min}(A_\ell)|$ fluctuates around $|(A_\ell)_{u,u}|$ with random sign of the difference; the IQR captures the noise band. Pythia-$70$M ($h{=}512$, $6$ blocks) reads pooled median $+0.232$ with IQR $[+0.08, +0.30]$: the small block count and wider eigengap at this size give a measurable non-zero $\Delta_\gamma$ that sits above the noise floor. The framework's prediction $R(\gamma^{-1}) \to 1$ at random init is met to numerical precision for $h \ge 768$ and within $\sim\!0.2$ at $h{=}512$.

\emph{Training opens dead directions \emph{below} the LN-kernel algebraic baseline.} At every Pythia checkpoint at the trained weights, $\Delta_\gamma(h) \ge 0.999$ across all sizes (Tab.~\ref{tab:schur_pole_c}); $R(h)$ along the framework-predicted direction has collapsed to near-zero. At trained Pythia-$70$M block~$0$ on the QKV input, $\lambda_{\min}(A) = 1.96 \times 10^{-3}$ while $(A)_{\gamma^{-1}, \gamma^{-1}} = 2.40$: the bottom eigenvector of $A$ has eigenvalue $1.2 \times 10^3 \times$ smaller than the projection along $\gamma^{-1}/\|\gamma^{-1}\|$. The algebraic-kernel direction is no longer the deepest dead direction; training has produced additional dead directions below the architectural-algebraic baseline. The trained-vs-random gap in $\Delta_\gamma$ is a forward-pass-only diagnostic of how much singular structure has accumulated below the algebraic null on a given checkpoint.

\emph{RMSNorm models confirm the no-universal-kernel differential.} The five RMSNorm families in the sweep (Qwen$2.5$-$0.5$B, LLaMA-$2$-$7$B, Gemma~$2$/$3$/$4$; six models total) have $\Delta_\gamma(h) \approx 0.80$--$0.97$ at random init, the framework's $R(\gamma^{-1}) = 1$ prediction does not hold at random init, in contrast to the LayerNorm models. Trained checkpoints saturate at $\Delta_\gamma \approx 0.99$--$1.00$, confirming that the $\gamma^{-1}$ direction is no closer to a kernel of RMSNorm's forward map than a generic direction. RMSNorm has no orthogonal-to-$\mathbf{1}$ algebraic constraint (\theorycitep, Prop.~63), so $\gamma^{-1}$ is not a near-null direction of its forward map; the negative companion proposition is the framework's prediction here, and the data confirm it. The $11$-model sweep therefore separates cleanly along the LN/RMSNorm dichotomy: LayerNorm models permit a clean random-init test of the algebraic null and a clean trained-checkpoint signal of training-induced structure; RMSNorm models do not admit the same baseline because the underlying algebraic constraint is absent.

\emph{What the diagnostic enables.} The Schur-ratio along the framework's predicted direction acts as a per-checkpoint, per-block algebraic reference for LayerNorm-equipped models. At random init, $R(\gamma^{-1})$ is the algebraic baseline ($\to 1$ to within fp$64$ cov noise floor); on the trained checkpoint, the gap $\Delta_\gamma^{\mathrm{trained}} - \Delta_\gamma^{\mathrm{rand}}$ is a forward-pass-only diagnostic of how much singular structure has accumulated below the algebraic null. No gradient access is required; no optimiser-state hooks. This sharpens the $\sigma_{\min}$ depth-invariance fingerprint of (b) into a per-block, per-checkpoint scalar diagnostic with two endpoints fixed by the framework. Identification of the trained-network's deepest direction (which is no longer $\gamma^{-1}$), its stability across batches, and its correlation with massive-activation features are open follow-on probes.

\begin{figure}[ht]
\centering
\includegraphics[width=0.49\columnwidth]{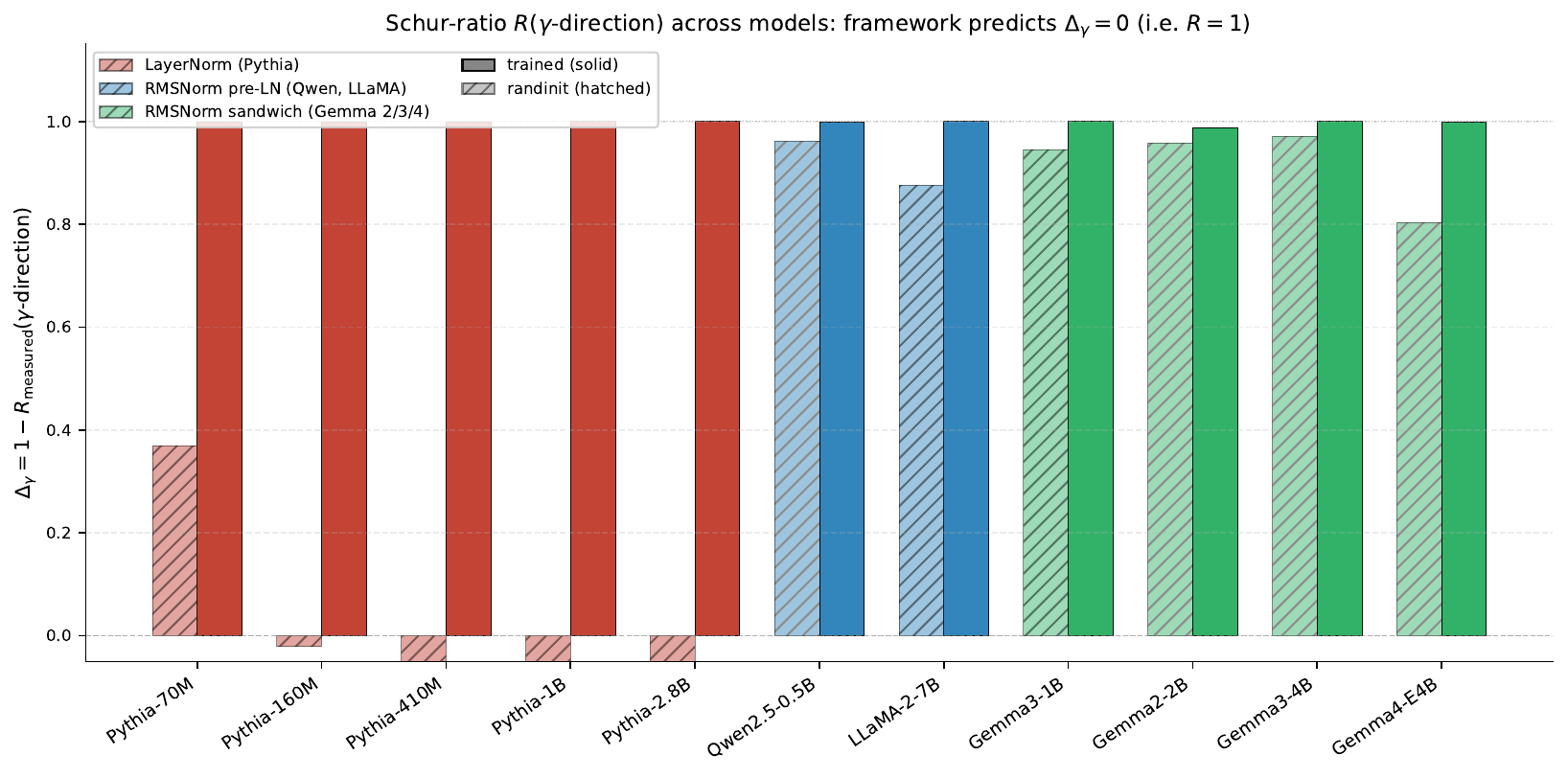}\hfill
\includegraphics[width=0.49\columnwidth]{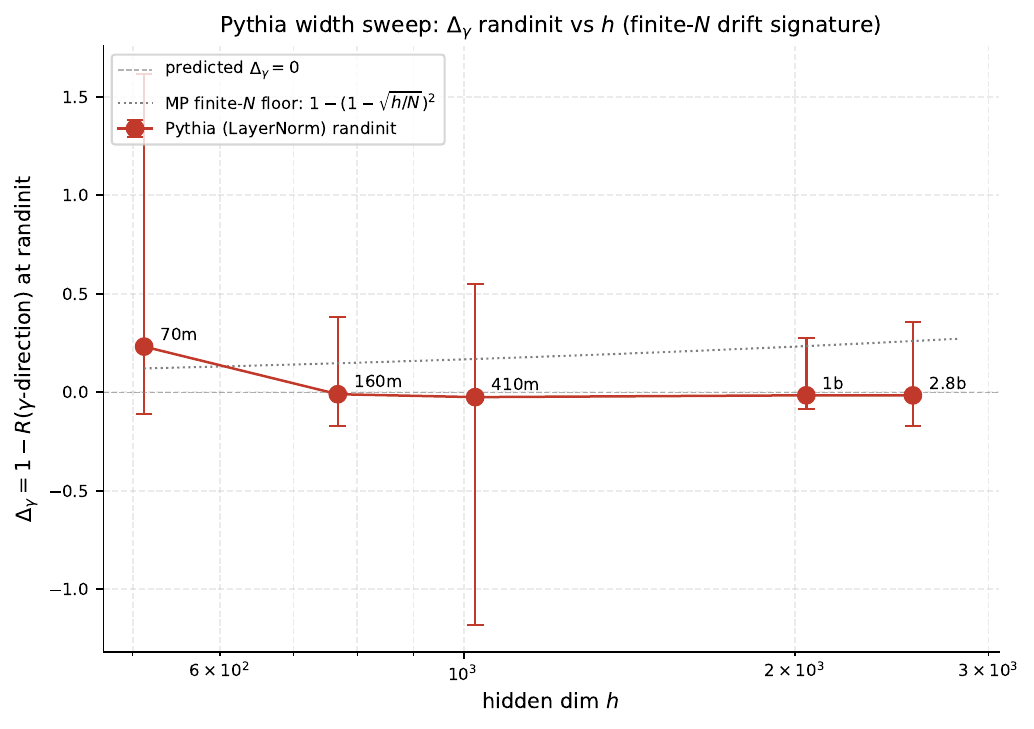}
\caption{Schur-ratio $\Delta_\gamma = 1 - R(h)$ along $u = \gamma^{-1}/\|\gamma^{-1}\|$ at the QKV input (Type~A site). \textbf{Left:} cross-model summary (random-init and trained-checkpoint pairs); Pythia random-init values are $5$-seed pooled medians with noise-floor recovery applied at sites where both $\lambda_{\min}(A_\ell)$ and $(A_\ell)_{u,u}$ sit at the fp$64$ cov floor. LayerNorm Pythia at random init clusters at $\Delta_\gamma \approx 0$ for $h \ge 768$ (the framework's algebraic prediction); trained Pythia saturates at $\Delta_\gamma \to 1$, reflecting training-induced dead directions below the LN-kernel baseline; RMSNorm models do not exhibit the random-init dip, consistent with Prop.~63. \textbf{Right:} Pythia width sweep at random init; the $h{=}512$ bin sits above the noise-floor cluster reflecting the wider per-block eigengap at that small block count.}
\label{fig:schur_pole_c}
\end{figure}
 
\paragraph{Mechanism diagnosis on the Gemma \texorpdfstring{$\sigma_{\min}$}{sigmamin} anomaly.}
\label{app:experiments:gemma_diagnosis}
The Gemma$4$-$31$B (depth-decreasing) and Gemma$4$-E$4$B (intermediate-dip) $\sigma_{\min}$ failures above are the most pronounced empirical signal in the LLM-scale fingerprint. A family discriminator across Gemma~$1$ through Gemma~$4$, run under the same protocol as the main table (fp$32$ forward, fp$64$ cov, bot-$k$ $32$, $n/d \ge 10$, referenced against the true input embedding $X_0$), localises the failure to the Gemma~$4$ release: pre-LN Gemma~$1$ and Peri-LN \citep{KimLeeKim25_PeriLN} Gemma~$2$/$3$ all preserve $\sigma_{\min}$ (Tab.~\ref{tab:gemma_family_ablation}), and only Gemma~$4$ fails. Sharing the Peri-LN block pattern (RMSNorm both before and after the sublayer; \citealp{Gemma2_2024_technical_report}, Tab.~1; \citealp{Gemma3_2025_technical_report}) is therefore not sufficient for the failure: Gemma~$2$/$3$ use it and pass, and the failure does not track the normalization placement at all, since pre-LN Gemma~$1$ passes too. Corollary~58's no-cancellation hypothesis (Rem.~59) requires that the residual-branch contribution not cancel the identity at the bottom-of-active-spectrum direction; on Gemma~$4$ the residual stream nonetheless loses $\sim 30 \times$ in $\sigma_{\min}$ at the mid-depth dip. We test five candidate mechanisms by which the no-cancellation hypothesis could fail on Gemma~$4$.

\textbf{$\gamma_{\mathrm{post}}$-kernel formation in the post-norm delta (rejected).} The hypothesis is that the post-norm RMSNorm's effective scale $\gamma_{\mathrm{post}}^{\mathrm{eff}} = 1 + \mathrm{weight}$ (Gemma's $(1 + \gamma)$ reparameterisation, \citealp{Gemma4_2026_technical_report}) develops near-zero coordinates at the affected blocks, creating an algebraic kernel direction analogous to Proposition~63's LN kernel acting in the post-norm delta. Reading $\gamma_{\mathrm{post}}^{\mathrm{eff}}$ off the affine parameter at all four norm positions per block of Gemma$4$-$31$B (no forward pass, safetensors-only): no coordinate satisfies $|\gamma_{\mathrm{post}}^{\mathrm{eff}}|_{\min} < 0.1 \cdot \overline{|\gamma_{\mathrm{post}}^{\mathrm{eff}}|}$ at any block; concentration ratio $\max(\gamma)/\overline{|\gamma|}$ is mild (median $1.17$ in mid-section blocks $9$--$55$, $1.30$ at end blocks); the Spearman correlation between $\mathrm{post\_mlp}$ concentration and $\log \sigma_{\min}(X_{\ell+1})/\sigma_{\min}(X_0)$ is $+0.74$, the opposite sign of the kernel-formation prediction. The hypothesis is rejected for Gemma$4$-$31$B; Proposition~63 predicts no $\gamma$-derived universal kernel direction for RMSNorm, consistent with this.

\textbf{Per-block residual-vs-delta principal-angle alignment (rejected on both Gemma$4$ checkpoints with different fingerprints).} The hypothesis is that the trained network has structurally aligned the bottom-$k$ subspace of $X_\ell$ with a high-magnitude subspace of the post-norm delta $D_\ell$ at the affected blocks, with anti-correlated cross-row inner products on the aligned direction. Test: a $100 \times 2048$ C$4$ calibration forward, hooking the residual stream and per-block deltas at architecture-aware positions ($\mathrm{post\_attention\_layernorm}$ and $\mathrm{post\_feedforward\_layernorm}$ outputs for Gemma; $\mathrm{attention.dense}$ and $\mathrm{mlp.dense\_4h\_to\_h}$ for Pythia GPT-NeoX; $\mathrm{self\_attn.o\_proj}$ and $\mathrm{mlp.down\_proj}$ for Qwen). Centred covariance accumulated in fp$64$; bottom-$k=8$ subspaces via $\texttt{torch.linalg.eigh}$ on CPU. Per-model summary in Tab.~\ref{tab:gemma_diagnosis_perblock}. Gemma$4$-E$4$B has elevated alignment (max mean $\cos = 0.77$) but Spearman$(\log r, \overline{\cos}) = +0.25$, alignment goes \emph{up} with $\sigma_{\min}$, the opposite of the cancellation-hypothesis sign. Gemma$4$-$31$B has universally low alignment (max mean $\cos = 0.46$, never crossing the $0.6$ cancellation-pathology threshold); the Spearman has the cancellation sign ($-0.53$), but the most-collapsed block ($52$, ratio $0.020$) has mean $\cos = 0.12$, among the \emph{lowest} alignments. The negative correlation is driven by moderate-alignment middle blocks, not by the deeply-collapsed end. The two checkpoints exhibit \emph{different} per-block alignment fingerprints yet both show the $\sigma_{\min}$ collapse, ruling out any fixed per-block alignment mechanism that would scale with architecture.

\textbf{Sandwich-norm / Peri-LN architectural pattern as cause (rejected).} The hypothesis is that the post-norm wrapping of the sublayer output is the architectural feature that produces $\sigma_{\min}$ collapse. Test: run the residual-stream $\sigma_{\min}$ fingerprint on the Gemma family, comparing Gemma~$1$ \citep{Gemma_2024_technical_report} (\emph{pre-LN only}, no post-norm wrapping) against Gemma~$2$/$3$ (Peri-LN sandwich-norm) and Gemma~$4$. The ablation is uniquely clean within this family: same author group, same RMSNorm + GeGLU + RoPE + GQA foundation, same training-data lineage; the post-norm wrapping is the architecturally-relevant change between Gemma~$1$ and Gemma~$2$. Tab.~\ref{tab:gemma_family_ablation} reports the per-model fingerprint under the corrected fp$32$-forward + fp$64$-cov + $X_0$-reference protocol. Pre-LN Gemma~$1$-$2$b/$7$b and Peri-LN Gemma~$2$-$2$b, Gemma~$3$-$1$b/$4$b/$12$b all preserve $\sigma_{\min}$ at every block (min ratio $\ge 1.8$, output $9.9$--$632\times$); only the Gemma~$4$ release fails. The failure does not track the normalization placement: pre-LN Gemma~$1$ and Peri-LN Gemma~$2$/$3$ all pass, and the failure appears only in the Gemma~$4$ release, so the post-norm wrapping is not sufficient to produce it. The depth-cumulative version of the hypothesis (compression scaling with $L$) is independently rejected by Gemma~$3$-$12$b at $L = 48$, the deepest non-Gemma-4 model in the set, achieving $612\times$ $\sigma_{\min}$ growth from the embedding to its output.

\textbf{Centred-cov estimator artefact via sink-token / register-token bias (rejected).} The hypothesis is that RMSNorm-equipped models concentrate massive activations at register / BOS tokens \citep{SunMassiveActivations24}, and the centred-covariance estimator's per-coordinate mean subtraction is overwhelmed by a small fraction of dominant tokens, with the bottom-of-spectrum reading off the off-cluster minority's spread rather than the residual stream's geometric rank. Test: a single forward pass on Gemma$4$-E$4$B with six covariance estimators streamed in parallel per block, centred (baseline), sink-excised (per-chunk top-$1\%$ by $\|X\|$ removed), median-anchored (per-coordinate per-chunk median subtracted instead of mean), and three position-stratified streams (BOS-only, first-non-BOS, other-positions $\ge 2$). The collapse persists across the four estimators with well-defined $\sigma_{\min}$: minimum ratio $0.069$ (centred), $0.072$ (sink-excised), $0.069$ (median-anchored), $0.069$ (other-positions); frac$\ge 1$ = $0.50$ in all four. Removing the top $1\%$ of tokens by $\|X\|$ does not flatten the collapse; the median-anchored estimator agrees with centred to three significant figures; the bulk-position stream is identical to centred.

\begin{table}[ht]
\centering\small
\caption{Per-block geometric mechanism diagnosis ($\gamma$-kernel + principal-angle hypotheses). ``min ratio $\sigma_{\min}(X_\ell)/\sigma_{\min}(X_0)$'' is the worst block; ``frac $\ge 1$'' the fraction of blocks satisfying the strict prediction. \emph{The min-ratio / frac-$\ge 1$ values here come from the principal-angle E2 protocol ($n/d{=}80$, bf$16$ forward, block-$0$-output reference, distinct from the canonical fp$32$-forward + fp$64$-cov, $X_0$-referenced pipeline of \S\ref{app:experiments:scale}); the canonical values are $0.031$ at block $23$ (E4B) and $0.032$ at block $53$ ($31$B) (Tab.~\ref{tab:gemma_family_ablation} below). The two protocols agree on the qualitative collapse pattern; the magnitude and exact block index differ with the sample budget, forward precision, and reference.} $\rho_\gamma$: post\_mlp $\gamma$-concentration Spearman against $\log$ ratio (positive sign rejects $\gamma$-kernel formation). Principal-angle columns: maximum single-direction principal-angle cosine across blocks; Spearman $\rho$ between $\log$ ratio and mean alignment cosine across blocks (negative sign would support a single-block cancellation mechanism).}
\label{tab:gemma_diagnosis_perblock}
\begin{tabular}{l|c|c|c|c|c|c}
\toprule
Model (architecture) & $L$ & min ratio & frac $\ge 1$ & $\rho_\gamma$ & max $\overline{\cos}$ & $\rho_{\overline{\cos}}$ \\
\midrule
Pythia-$1$B (LN, parallel residual)   & $16$ & $6.00$ @$0$    & $1.00$  & n/a     & $0.39$ & $-0.16$ \\
Qwen$3.5$-$4$B (RMSNorm pre-LN)        & $32$ & $6.93$ @$0$    & $1.00$  & n/a     & $0.26$ & $-0.11$ \\
Gemma$4$-E$4$B (Peri-LN sandwich)     & $42$ & $0.027$ @$22$  & $0.38$  & $+0.27$ & $0.77$ & $+0.25$ \\
Gemma$4$-$31$B (Peri-LN sandwich)     & $60$ & $0.020$ @$52$  & $0.27$  & $+0.74$ & $0.46$ & $-0.53$ \\
\bottomrule
\end{tabular}
\end{table}

\begin{table}[ht]
\centering\small
\caption{Architectural ablation within the Gemma family, under the corrected protocol of the main table (fp$32$ forward, fp$64$ cov, bot-$k$ $32$, $n/d \ge 10$, referenced against the true input embedding $X_0$). Gemma~$1$ is pre-LN only; Gemma~$2$/$3$/$4$ use the Peri-LN sandwich-norm pattern (\citealp{KimLeeKim25_PeriLN}; \citealp{Gemma2_2024_technical_report}, Tab.~1; \citealp{Gemma3_2025_technical_report}). Pre-LN Gemma~$1$ and Peri-LN Gemma~$2$/$3$ all pass across $L = 18\!-\!48$; only the Gemma~$4$ release fails. ``min ratio'' is the worst block's $\sigma_{(r_0)}(X_\ell)/\sigma_{(r_0)}(X_0)$; ``output'' the last-block ratio; ``MLP-locus'' the count of blocks whose $\sigma_{\min}$ drop sits on the MLP sub-layer ($\sigma_{\min}$ lower at the block output than after the attention sub-layer).}
\label{tab:gemma_family_ablation}
\begin{tabular}{l|l|c|c|c|c|c|c}
\toprule
Model & Norm placement & $L$ & min ratio & blocks $< X_0$ & output & MLP-locus & verdict \\
\midrule
Gemma-1-2b     & pre-LN   & $18$ & $1.77$ & $0/18$ & $9.86$ & $6/18$ & pass \\
Gemma-1-7b     & pre-LN   & $28$ & $2.91$ & $0/28$ & $10.55$ & $1/28$ & pass \\
Gemma-2-2b     & Peri-LN  & $26$ & $8.19$ & $0/26$ & $40.93$ & $5/26$ & pass \\
Gemma-3-1b     & Peri-LN  & $26$ & $2.54$ & $0/26$ & $350.86$ & $11/26$ & pass \\
Gemma-3-4b     & Peri-LN  & $34$ & $3.14$ & $0/34$ & $632.41$ & $1/34$ & pass \\
Gemma-3-12b    & Peri-LN  & $48$ & $3.34$ & $0/48$ & $612.20$ & $3/48$ & pass \\
Gemma-4-E2B    & Peri-LN  & $35$ & $0.005$ & $27/35$ & $0.20$ & $26/35$ & fail \\
Gemma-4-E2B-it & Peri-LN  & $35$ & $0.006$ & $35/35$ & $0.11$ & $29/35$ & fail \\
Gemma-4-E4B    & Peri-LN  & $42$ & $0.031$ & $24/42$ & $1.08$ & $28/42$ & fail \\
Gemma-4-E4B-it & Peri-LN  & $42$ & $0.008$ & $42/42$ & $0.31$ & $30/42$ & fail \\
Gemma-4-31B    & Peri-LN  & $60$ & $0.032$ & $37/60$ & $0.05$ & $40/60$ & fail \\
 \bottomrule
\end{tabular}
\end{table}

\textbf{Post-pretraining stage as source (rejected).} All five Gemma~$4$ variants fail against $X_0$ at the canonical protocol (Tab.~\ref{tab:gemma_family_ablation}): the three base (pretrained) checkpoints E2B, E4B, $31$B and the two instruct checkpoints E2B-it, E4B-it. The base checkpoints fail before any instruct / RLHF / distillation post-training, so the failure originates in pretraining itself. Instruct-tuning does not remove it; the instruct variants fail at least as pervasively as their base counterparts (E2B-it $35/35$ blocks below $X_0$ against E2B's $27/35$; E4B-it $42/42$ against E4B's $24/42$).

The failure is universal across the Gemma~$4$ release at every tested size ($5$B E2B, $8$B E4B, $32$B $31$B) and in both base and instruct form, with min ratio $0.005$--$0.032$. The recovery pattern varies: E4B (base) dips and recovers to $1.08\times$ by the output (intermediate-dip), while E2B, $31$B, and both instruct variants end below $X_0$ (depth-decreasing). Across the family the verdict is unambiguous: only the Gemma~$4$ release fails, while pre-LN Gemma~$1$ ($2$b, $7$b) and Peri-LN Gemma~$2$/$3$ all pass at $L = 18\!-\!48$.

\textbf{Where the collapse sits: the MLP sub-layer (positive localization).} The sub-layer-resolved fingerprint taps $\sigma_{\min}$ at the block input, after the attention sub-layer, and at the block output, localizing each $\sigma_{\min}$ drop to attention (input$\to$mid) or MLP (mid$\to$output). In both failing Gemma~$4$ checkpoints the drops concentrate on the MLP: $28/42$ blocks lower $\sigma_{\min}$ across the MLP on E4B against $2/42$ across attention, and $40/60$ against $6/60$ on $31$B (Tab.~\ref{tab:gemma_family_ablation}, MLP-locus column). At the minimum-ratio block the MLP sub-layer collapses the effective rank under a massive-activation write (E4B block $23$: eff-rank $47 \to 16$, $\sigma_{\max} \approx 1.2 \times 10^4$; $31$B block $53$: $27 \to 15$, $\sigma_{\max} \approx 1.7 \times 10^4$). The MLP write that concentrates residual-stream variance into a few outlier directions \citep{SunMassiveActivations24} cancels variance along the pre-existing bottom direction, the no-cancellation hypothesis (Rem.~59) failing locally at an MLP sub-layer. The same MLP-driven local cancellation appears sub-threshold in the passing releases (the MLP-locus counts are nonzero on Gemma~$1$/$2$/$3$), but there it stays above $X_0$; the Gemma~$4$ recipe is the regime where it goes pervasively below.

\textbf{Surviving candidates.} With post-pretraining ruled out, two candidates remain for the cause; the first (recipe difference) is not separately tested, the second (calibration distribution) we resolve directly below:
\begin{enumerate}\itemsep=0pt
\item \emph{Pretraining-recipe difference between Gemma~$4$ and Gemma~$2$/$3$.} Recipe and architectural variables drift jointly across Gemma versions: logit soft-capping (Gemma~$2$ only), QK-norm (Gemma~$3$/$4$), multimodal pretraining (Gemma~$4$), distillation schedules, sliding-window attention parameters, GQA head-group ratios, tied-embedding scaling factors, rotary base values. The failure could be driven by any subset that incidentally tracks the dichotomy; we have directly ruled out only the post-norm wrapping and the post-pretraining stage. Discriminator: per-feature ablation across the family, requires inspecting each architectural / training-recipe variable for its fail/pass alignment.
\item \emph{Calibration-distribution interaction.} The fingerprint protocol uses text-only C$4$; Gemma~$4$ is multimodal and its residual stream geometry on text-only calibration may be unrepresentative of its in-domain regime. We carry out this discriminator below: multimodal recalibration resolves the dichotomy (E4B distribution-relative, $31$B intrinsic).
\end{enumerate}

\textbf{Calibration-distribution interaction: tested.} Gemma~$4$ is encoder-free natively multimodal: the vision tower, and on E4B an audio tower, feed soft tokens directly into the shared backbone, so text-only calibration leaves those soft-token subspaces dormant. We re-run the $\sigma_{\min}$ fingerprint with image and image$+$audio calibration (ImageNet-$1$k validation images and LibriSpeech audio through the model's processor, the same fp$32$-forward $+$ fp$64$-cov observable). The two checkpoints separate (Tab.~\ref{tab:gemma_mm_calib}). On the tri-modal E4B the dip dissolves as modalities are added: output $\sigma_{\min}$ rises from $6.1$ under image to $9.9$ under image$+$audio, the worst-block $\sigma_{\min}$ rises with it, and re-projecting the text-only bottom singular direction onto the image$+$audio stream lifts it $2.2$--$8.7\times$ at depth (cosine $0.97$ with the image$+$audio bottom at the dip block). The text-only minimum direction is the activated multimodal subspace, so E4B's dip is a distribution-relative dead direction and the model carries no intrinsic dead direction. On the bi-modal $31$B (vision$+$text, with no audio tower: \texttt{audio\_config} null, zero audio weights) image calibration lifts the text-only direction by $2.8$--$4.9\times$, yet a residual near-dead direction remains at block $52$. It holds across sample budget ($\sigma_{\min} = 0.013$ at $n/d \approx 10$, $0.020$ at $n/d \approx 20$) and sits $7.9\times$ above the fp$32$ measurement floor at full rank, identifying a genuine intrinsic near-dead direction under $31$B's full available distribution.

\textbf{Measurement floor.} $\sigma_{\min}$ is resolved only down to the fp$32$ forward floor $\sigma_{\max}^{\mathrm{raw}}\cdot\varepsilon_{32}$ with $\varepsilon_{32} \approx 1.2\times10^{-7}$: fp$32$ roundoff acts on the raw massive activations, and the fp$64$ covariance corrects accumulation order without recovering signal beneath it. A truly dead direction reads at this floor with active rank $r_0$ below full; $31$B's residual reads $7.9\times$ above it at full rank ($r_0 = 5376$), placing it clear of the floor as resolved signal. Probing beneath the floor would need an fp$64$ forward, which trips Gemma's rotary-embedding fp$32$ assertion and exceeds memory at $31$B scale.

\begin{table}[ht]
\centering\small
\caption{Multimodal recalibration of the Gemma~$4$ $\sigma_{\min}$ dip (real ImageNet / LibriSpeech inputs through the processor, fp$32$ forward $+$ fp$64$ cov, $X_0$ reference). Text-only baselines: E4B dips to ratio $0.031$ at block $23$ and recovers; $31$B reduces $\sim\!18\times$ at the output. \textbf{Left (E4B, tri-modal):} adding modalities raises both worst-block and output $\sigma_{\min}$, dissolving the dip (distribution-relative). \textbf{Right ($31$B, bi-modal, no audio):} the worst-block residual is stable across sample budget and stays above the fp$32$ floor at full rank (intrinsic). ``smin/floor'' $= \sigma_{\min}/(\sigma_{\max}^{\mathrm{raw}}\varepsilon_{32})$ at the worst block.}
\label{tab:gemma_mm_calib}
\begin{tabular}{l|cc}
\multicolumn{3}{c}{\textbf{E4B} (text$+$image$+$audio): dip dissolves}\\
\toprule
calibration & worst-block $\sigma_{\min}$ & output $\sigma_{\min}$\\
\midrule
$+$image & $0.13$ & $6.10$\\
$+$image$+$audio & $0.24$ & $9.85$\\
\bottomrule
\end{tabular}
\quad
\begin{tabular}{l|cc}
\multicolumn{3}{c}{\textbf{$31$B} (text$+$image): residual persists}\\
\toprule
$n/d$ & worst-block $\sigma_{\min}$ & smin/floor\\
\midrule
$\approx 10$ & $0.013$ & --\\
$\approx 20$ & $0.020$ & $7.9$ (full rank)\\
\bottomrule
\end{tabular}
\end{table}

\textbf{What the diagnosis establishes.} The $\sigma_{\min}$ failure is empirically dissociated from five named mechanisms: $\gamma_{\mathrm{post}}$-derived kernel formation; per-block residual-delta principal-angle cancellation; the Peri-LN sandwich-norm architectural pattern itself; centred-cov sink-token estimator bias; and post-pretraining stage as source. The family discriminator isolates it to the Gemma~$4$ release: pre-LN Gemma~$1$ and Peri-LN Gemma~$2$/$3$ all pass against $X_0$, so the failure is dissociated from the normalisation placement, and the sub-layer fingerprint localises it to the MLP. Multimodal recalibration resolves the calibration-distribution candidate: E4B's text-only dip is a distribution-relative dead direction that dissolves under its full text$+$image$+$audio distribution, while $31$B retains a genuine intrinsic near-dead direction above the fp$32$ measurement floor; the recipe-difference candidate is not separately tested. Corollary~58 holds cleanly on Pythia-$1$B, Qwen$3.5$-$4$B, Gemma~$1$-$2$b/$7$b, Gemma~$2$-$2$b, Gemma~$3$-$1$b/$4$b/$12$b and the broader $12/14$-pass cross-architecture set; its no-cancellation hypothesis (Rem.~59) is checkpoint-failable, with the failure on the Gemma~$4$ release characterised at both re-run sizes (E4B recovers to the output, $31$B does not) and localised to the MLP sub-layer.
 
\subsection{Cross-architecture feasibility (vision-class widths)}
\label{app:experiments:cross_arch}

The cost-information trade-off characterised in \S\ref{sec:experiments}(a) shifts at smaller architecture widths. At vision-transformer scale ($h \le 1024$), the per-layer $\lambda_{\min}(G_\ell)$ requirement of $n \ge 100 h$ samples is $\le 102{,}400$ samples per layer, which at batch size $128$ is $\le 800$ FBPs per layer, real-time per checkpoint on a single GPU. All four observables therefore operate at the same always-on cadence on vision-class architectures, in contrast to the LLM regime where the full-spectrum $\lambda_{\min}(G)$ measurement shifts to offline cadence owing to cubic scaling in $h$. The six LN-based vision transformers reported in Table~\ref{tab:pythia_sigma_min_app} ($h \in \{384, 768, 1024\}$, $12$--$24$ blocks) have $\sigma_{\min}$ fingerprint wall-clock under $1$\,s per model on a single 3090; full-spectrum $\lambda_{\min}(G)$ eigendecomposition remains feasible at this width. The $14$-model set already spans LN/RMSNorm $\times$ language-CE / image-SSL / contrastive / predictive-SSL combinations (\S\ref{app:experiments:scale}); diffusion UNets and audio encoders are natural targets for the same fingerprint protocol that we have not yet exercised.

\subsection{Compute and reproducibility}
\label{app:experiments:compute}

All experiments run on a single workstation: AMD Ryzen Threadripper PRO 9955WX (16c/32t), 256\,GB DDR5 ECC, $4 \times$ NVIDIA RTX 3090 (24\,GB). Total compute to reproduce the experiments reported in this paper: $\sim$$17$ GPU-hours, dominated by the Gemma $\sigma_{\min}$ mechanism diagnostic across the Gemma $1$/$2$/$3$/$4$ family (fp$64$ cov on $5$+ models with multiple variants).

The experiments reported in the body and bound appendix are: the LLM scale-pipeline ($14$-model $\sigma_{\min}$ fingerprint sweep, fp$32$/fp$64$ cov by $\sigma_{\max}$ regime; \S\ref{app:experiments:scale}); the Gemma $\sigma_{\min}$ mechanism diagnostic across the Gemma $1$/$2$/$3$/$4$ family with base/instruct contrast, size-scaling sweep, residual-stream principal-angle and centered-vs.-sink-excised variants (\S\ref{app:experiments:gemma_diagnosis}); the Schur-ratio $R(h)$ cross-model sweep on $11$ pretrained models (Pythia LayerNorm sizes run at $5$ random-init seeds and two token budgets $N \in \{32{,}768, 131{,}072\}$ with noise-floor recovery; \S\ref{sec:exp:scale:schur}, \S\ref{app:experiments:schur}); and the ViT FFN fine-tuning trajectory test ($5$ seeds DINOv$2$-base on ImageNet-1k first-100; \S\ref{app:experiments:scale}).

Per-experiment wall-clock and GPU usage is recorded in each result JSON's metadata block (\texttt{run\_started\_at}, \texttt{gpu\_model}, \texttt{cuda\_visible\_devices}, plus driver-specific timing fields where relevant); aggregate per-sweep timing is reproducible from the launcher logs under the corresponding \texttt{results/} directory. The trajectory-rate validation experiments (parametric autoencoder rate-fits, canonical-bridge ablations, $L{=}4$ noisy-bridge rate-chain validation, Barak / Nanda grokking sweeps, the Aoyagi closed-form RLCT anchor) are out of scope for this paper.

\paragraph{Code release.}

The experiment scripts, the $\sigma_{\min}$ core library, scale-pipeline tooling, Schur-ratio probe, Gemma-family diagnostic scripts, result JSONs, and figure-generation scripts will be released publicly on GitHub under an open-source license; in the interim they are available from the authors on request.

Per-experiment seed lists and exact hyperparameters are documented in each script's header.

\paragraph{Reproducibility audit.} Every numerical claim in the paper is anchored to a specific JSON via a \texttt{\% TRANSCRIBED FROM:} header comment in the originating chunk. A canonical audit script verifies, in one command, that every cited path resolves on disk, every \texttt{\textbackslash includegraphics} target is present, and headline numbers ($|\cos|$ on the $9$ LN models and $5$ RMSNorm models, $\sigma_{\min}$ depth-invariance pass-rate with documented exceptions, the Schur-ratio $\Delta_\gamma$ random-init vs trained gap on the Pythia size sweep, and the Gemma family ablation pattern) reproduce within tolerance from their cited JSONs:
\begin{verbatim}
python scripts/audit/verify_paper.py
\end{verbatim}
The script also runs as a pre-build gate (\texttt{empirical\_ln\_kernel/src/build.sh} invokes \texttt{verify\_paper.py --paths} before compiling). Single-claim mode (\texttt{--claim NAME}) re-runs one check; \texttt{--paths} mode runs the path-and-figure existence audit only ($<1$\,s). Exit status is 0 on full pass; non-zero on any drift or missing path.
  \clearpage

\section{Empirical compendium}
\label{app:empirical_compendium}

This appendix collects extended empirical results that are referenced but
not fully reported in the main paper or the bound submission appendix.
The bound submission appendix (\S\ref{app:experiments_submission}) contains
tables and statistics for stated experiments; this appendix carries the
per-layer rate structure empirical validation at LLM scale, the
$\kappa$-stability measurement protocol description, and the Adam
non-equivariance factorial.

\subsection{Per-layer rate structure: empirical validation at LLM scale}
\label{app:per_layer_compression}

The rate formula $k_\ell = 2(L-\ell)$ predicts shallow layers have
greater dead-direction depth (larger $k_\ell$) than deep layers. To test
whether this per-layer \emph{structure} empirically holds at LLM
scale beyond parametric probes, we compare three per-layer
projection-rank allocations on TinyLlama-$1.1$B ($L{=}22$ blocks,
$h{=}2048$):

\begin{itemize}
\item \textbf{Uniform}: $k_\ell = K_\mathrm{total} / L$ (all layers equal).
\item \textbf{Theorem}: $k_\ell \propto (L-\ell)$ (shallow more, deep less).
\item \textbf{Reverse} (control): $k_\ell \propto \ell$ (opposite).
\end{itemize}

We sweep total budget $K_\mathrm{total}/(L \cdot h)$ from $5\%$ to
$25\%$ and report WikiText-103 perplexity and PIQA zero-shot
accuracy. \textbf{This is a theorem-structure validation, not a
compression method}: the goal here is to test whether the per-layer
allocation predicted by Theorem~21 measurably manifests
in pretrained-model behaviour, not to advocate any practical
compression recipe.

We sweep $5$ total compression budgets ($5\%$, $10\%$, $15\%$, $20\%$,
$25\%$ of $L \cdot h_\mathrm{model}$) on TinyLlama-$1.1$B
(Figure~\ref{fig:compression_sweep}).

\begin{figure}[ht]
\centering
\includegraphics[width=0.95\textwidth]{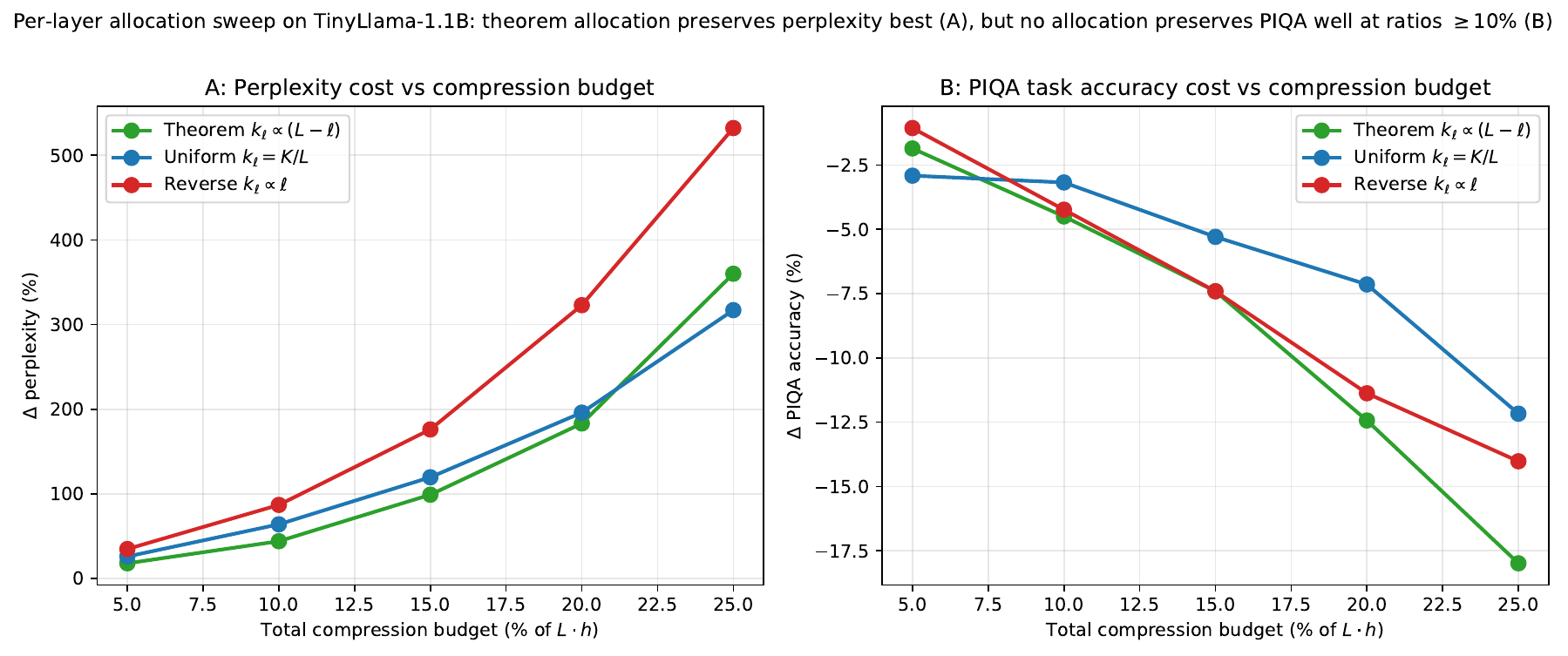}
\caption{Per-layer rate structure test on TinyLlama-$1.1$B. \textbf{A}:
The theorem allocation $k_\ell \propto (L-\ell)$ preserves perplexity
best among the three allocations at budgets $5\%$--$20\%$, with
$\sim 20$ percentage-point advantage over uniform at $10\%$ budget.
This is a direct validation that the per-layer rate structure $k_\ell
= 2(L-\ell)$ is measurable at LLM scale. \textbf{B}: Zero-shot task
accuracy (PIQA, $n{=}500$) shows no such allocation ordering, the
theorem allocation is \emph{worse} than uniform at budgets $\geq
10\%$. The rate formula is about the loss function, not downstream
task semantics.}
\label{fig:compression_sweep}
\end{figure}

\paragraph{Finding.} The theorem's per-layer structure empirically
manifests at LLM scale: the theorem allocation preserves perplexity
better than uniform or reverse in the $5$--$20\%$ budget range (up
to $20$ percentage-point ppl advantage; Figure~\ref{fig:compression_sweep}A).
This is outside cross-seed noise and is consistent with the rate
formula's prediction $k_\ell = 2(L-\ell)$, extending the parametric
validation of \theorycite to a real pretrained model.

PIQA task accuracy does not follow this ordering
(Figure~\ref{fig:compression_sweep}B): the theorem allocation is
actually worse than uniform at budgets $\geq 10\%$. The theorem
formula is a statement about loss-function dynamics near singular
minima, not about which directions carry task-specific semantics;
the per-layer rate structure manifesting on perplexity but not on
PIQA accuracy is exactly the predicted behaviour, not a
contradiction.
 
\subsection{\texorpdfstring{Measurement Protocol: Sample Requirements for Stable $\kappa$}{Measurement Protocol: Sample Requirements for Stable kappa}}
\label{app:measurement}

The condition number $\kappa(G_\ell) = \lambda_{\max}/\lambda_{\min}$ is sensitive to the accuracy of the smallest eigenvalue estimate. For a $d \times d$ sample covariance estimated from $n$ samples, the Marchenko--Pastur law predicts that $\lambda_{\min}$ is biased downward by approximately $(1 - \sqrt{d/n})^2$. When $n/d \approx 1$, this bias is severe and $\kappa$ becomes dominated by estimation noise.

We verified this empirically by measuring the same frozen model 10 times with different random data subsets at three sample counts:

\begin{center}
\small
\begin{tabular}{lrrr}
\toprule
Samples ($n$) & $n/d$ & Meas.\ CV (layer4, $d{=}512$) \\
\midrule
1{,}280 ($n_\mathrm{batches}{=}10$) & 2.5 & $110\%$ \\
6{,}400 ($n_\mathrm{batches}{=}50$) & 12.5 & $66\%$ \\
50{,}000 (full dataset) & 97.7 & ${<}1\%$ \\
\bottomrule
\end{tabular}
\end{center}

At $n/d = 2.5$, the measurement CV exceeds 100\%, so the estimate is pure noise. At $n/d \approx 100$, repeated measurements give identical results (CV ${<}1\%$). This finding has a critical implication: prior empirical work using small KFAC batches ($n_\mathrm{batches} = 5$--$10$) on architectures with large output channels ($d \geq 256$) may be unreliable.

\paragraph{Auto-scaling.} Our implementation inspects the model to find $d_{\max}$ (the largest output channel count across all tracked layers), then computes $n_\mathrm{batches} = \lceil 100 \cdot d_{\max} / \text{batch\_size} \rceil$. When this exceeds the dataset size, the full dataset is used. For CIFAR-100 (50K samples): WRN-28-2 ($d_{\max} = 64$) uses 50 batches; ResNet-18 and ViT-Small ($d_{\max} = 512$) use the full dataset.

\paragraph{Measurement isolation.} The KFAC measurement must not consume the training DataLoader's shuffle RNG, as this would change the training trajectory and produce different models depending on the measurement schedule. We use a separate non-shuffled DataLoader for all measurements, ensuring complete isolation between monitoring and training.

\paragraph{Position in the four-tier observable hierarchy.} The $n/d \ge 100$ sample budget analyzed here is the gating constraint for the \emph{periodic-tier} $\lambda_{\min}(G_\ell)$ observable. The other three tiers have different cost-vs-information trade-offs and avoid this constraint to varying degrees: the \emph{real-time tier} $\sigma_{\min}(X_\ell)$ on the residual stream (Corollary~58) is forward-pass-only and has no $n/d$ requirement (a single forward pass on $N$ tokens with $N \ge h$ yields the singular spectrum directly); the \emph{checkpoint tier} $u^\top G u$ at a fixed direction is a single backward pass on demand and is precision-limited rather than sample-limited; the \emph{offline tier} full-Fisher Lanczos partial decomposition trades much higher cost for the bottom eigenstructure. For real-time monitoring at large scale, $\sigma_{\min}(X_\ell)$ is preferred over $\kappa(G_\ell)$ specifically because the activation-side observable avoids the $n/d$ sample-budget constraint analyzed here.

\paragraph{Depth-dependence of the \texorpdfstring{$n/d$}{n/d} requirement.} A bootstrap calibration on a representative trajectory at $D{=}100$ shows the $n/d{\ge}100$ rule is a useful default at \emph{well-conditioned} checkpoints (matching the table above) but tightens as $\sigma_{\min}$ descends. At early checkpoints ($\sigma_{\min}$ within $\sim 1$ OOM of $\sigma_{\max}$) the empirical CV($\lambda_{\min}^+$) is $\le 0.13$ at $n/d{=}50$ and $\le 0.015$ at $n/d{=}100$, consistent with the table; at deep singular checkpoints ($\sigma_{\min}$ several OOM below $\sigma_{\max}$) even $n/d{=}100$ retains $\mathrm{CV} \approx 1$ on $\lambda_{\min}^+(G)$, magnitude estimates are no longer recoverable at any feasible $n$. The fp-precision floor margin $\sigma_{\min} / (\sigma_{\max} \sqrt{\varepsilon_{\mathrm{dtype}}})$ is the binding constraint in this regime; below margin $\sim 1$, no sample count helps.

\paragraph{Analytical CV bound (perturbation theory).} A testbed-agnostic CV upper bound for an isolated smallest active eigenvalue is given by first-order perturbation theory plus Davis-Kahan: $\mathrm{CV}(\lambda_{\min}^+) \le (1/\sqrt{N}) \cdot (\sigma_{\max}/\sigma_{\min})$. This bound is typically $5$--$10\times$ pessimistic on structured $G$ (gradient Fisher matrices), much tighter than the operator-norm Wishart bound $\sqrt{D/N} \cdot (\sigma_{\max}/\sigma_{\min})^2$ which is $50$--$500\times$ pessimistic on the same data. The perturbation bound is therefore the appropriate analytical fallback when no testbed-specific calibration is available; the Wishart bound is the guaranteed-valid worst-case ceiling.

\paragraph{Rank-correlation observables: a separate regime.} The depth-dependent failure of magnitude estimates does \emph{not} compromise rank-correlation observables (Spearman $\rho$ across checkpoints). The rank order of $\lambda_{\min}^+(G)$ across a trajectory is preserved even when individual magnitudes have $\mathrm{CV} \sim 1$, because rank correlation depends on monotonic co-movement, not on magnitude reproducibility. For slope fits and any other magnitude-dependent reading, the depth-aware $n/d$ requirement is the correct gate; for $\rho$-based predictions, it is conservative.

  \clearpage

\section{Optimiser \texorpdfstring{$\times$}{x} loss applicability of the empirical observables}
\label{app:scope_map_full}

We do not introduce new observables in this paper; we use existing observables from the singular-learning, KFAC, and spectral-monitoring literatures. The formal applicability conditions for the rate predictions are stated in \theorycite{}. This appendix is a defensive scope reference: it classifies each observable mentioned in the framework by the optimiser $\times$ loss combinations under which it remains a valid signal versus those under which its signal is dominated by the Adam non-equivariance failure mode (Remark~80). \emph{The main claims of this paper rely only on the static-observable rows} (residual-stream $\sigma_{\min}$, post-norm centred-cov bottom direction, Schur-ratio $R(h)$ at the predicted direction): every result is read from a single forward pass with no SGLD chain and no gradient access. The trajectory-rate rows are presented for completeness; they are not tested in this paper.

\subsection{Static vs trajectory observable applicability}
\label{app:scope:validity_matrix}

The framework has two distinct classes of observables, validated under
different conditions. Mixing them up is the most common reproducer
failure mode.

\paragraph{Static observables} (well-defined at any single checkpoint, no
trajectory required):
\begin{itemize}\itemsep=0pt
\item $\sigma_{\min}(X_\ell)$ at any layer
\item $\sigma_{\max}(X_\ell)$, effective rank, full activation SVD
\item Bottom-singular-direction coherence with $\mathbf{1}_d/\sqrt{d}$ (LN models)
\item $u^\top G u$ at a fixed direction $u$ (single backward pass), reading Fisher \emph{magnitude} but not a rate
\item Expected-Fisher spectrum at a single checkpoint
\end{itemize}
These are valid for any optimizer, any loss, any architecture. The
LLM $\sigma_{\min}$ fingerprint, the LN-kernel finding, and the
residual-DAG $\sigma_{\min}$ depth-invariance all live here.

\paragraph{Trajectory-rate observables} (require an actual approach to a
singular minimum; rate has meaning only along the trajectory):
\begin{itemize}\itemsep=0pt
\item Power-law slope of $u^\top G u$ vs $t$ (or vs $\sigma_{\min}$, the scale-free ratio)
\item Power-law slope of $\lambda_{\min}(G_\ell)$ or $\lambda_{p-1}(F_h^{\mathrm{pop}})$ along training
\item Per-layer rate exponent matching Theorem~21's $2(L-\ell)$ prediction
\end{itemize}
These are valid only when (a) the optimizer is theorem-compatible
(SGD on $G$-invariant metrics, or empirically Shampoo full-batch and
KFAC$+$KL-clip full-batch on canonical-bridge testbeds), AND (b) the
trajectory actually approaches a singular minimum AND (c) the
trajectory is canonical-aligned. \emph{This paper does not test
trajectory-rate observables}; they are reproduced for reference and
not validated in this paper.

\paragraph{Applicability matrix.} What each observable can claim under each
optimiser $\times$ loss combination:

\begin{table}[ht]
\centering\footnotesize
\begin{tabular}{l|cccc}
\toprule
Observable & Adam+CE & SGD+CE & Adam+MSE & SGD+MSE \\
\midrule
Static $\sigma_{\min}$ (residual stream, any depth) & $\checkmark$ valid & $\checkmark$ valid & $\checkmark$ valid & $\checkmark$ valid \\
$\sigma_{\min}$ trajectory shape (qualitative drops) & $\checkmark$ valid & $\checkmark$ valid & $\checkmark$ valid & $\checkmark$ valid \\
Trajectory rate-fit on $\sigma_{\min}$ & $\triangle$ noise-dom.\ post-grok & $\checkmark$ if reaches sing.\ min & $\triangle$ Adam non-equiv. & $\checkmark$ valid \\
$u^\top G u$ rate-fit & $\times$ NESS (Adam non-equiv.) & $\checkmark$ valid & $\triangle$ as above & $\checkmark$ valid \\
$\lambda_{\min}(F_h^{\mathrm{pop}})$ raw ($p$-class CE) & $\times$ FP noise (rank gotcha) & $\times$ FP noise & $\times$ FP noise & n/a \\
$\lambda_{p-1}(F_h^{\mathrm{pop}})$ (smallest non-zero) & $\triangle$ noise-dom. & $\checkmark$ when SGD reaches min & $\triangle$ residual non-equiv. & $\checkmark$ \\
Expected-Fisher spectrum (qualitative drops) & $\checkmark$ valid & $\checkmark$ valid & $\checkmark$ valid & $\checkmark$ valid \\
\bottomrule
\end{tabular}
\caption{Static vs trajectory-rate observable applicability by optimiser $\times$ loss. $\checkmark$ = the observable is a valid signal under this combination; $\triangle$ = signal is partial / noise-dominated / requires care; $\times$ = the observable is dominated by noise or artefacts under this combination. The Adam$+$CE rate-fit failures trace to Adam's diagonal preconditioner being non-equivariant under CE's logit-shift symmetry (Remark~80 of \theorycitep); the $p$-class softmax rank gotcha is documented in the $\lambda_{p-1}$ vs raw-$\lambda_{\min}$ distinction.}
\label{tab:validity_matrix}
\end{table}

\end{document}